\documentclass[11pt]{article}

\usepackage{amsmath, amssymb, amsthm, bm}
\usepackage{fullpage}
\usepackage{natbib}
\usepackage{hyperref}
\usepackage{enumitem}
\usepackage{graphicx}
\usepackage{booktabs}
\usepackage{subcaption}
\usepackage{xcolor}
\hypersetup{
  colorlinks=true,
  linkcolor=blue,
  citecolor=blue,
  urlcolor=blue
}

\newcommand{\E}{\mathbb{E}}
\newcommand{\Var}{\mathrm{Var}}
\newcommand{\Cov}{\mathrm{Cov}}

\newcommand{\cN}{\mathcal{N}}

\newcommand{\hY}{\widehat{Y}}
\newcommand{\bmu}{\bm{\mu}} 
\usepackage{amsthm}
\newtheorem{proposition}{Proposition}

\newtheorem{lemma}{Lemma}

\theoremstyle{remark}

\begin{document}

\title{Efficient Inference for Noisy LLM-as-a-Judge Evaluation}
\author{
    Yiqun T. Chen\thanks{Departments of Biostatistics and Computer Science, Johns Hopkins University, Baltimore, MD 21205, USA; correspondence email: \texttt{yiqun.t.chen@gmail.com}}
    \and Sizhu Lu\thanks{Department of Statistics, University of California, Berkeley, CA 94720, USA}
    \and Sijia Li\thanks{Department of Biostatistics, University of California, Los Angeles, CA 90095, USA}
    \and Moran Guo\thanks{Department of Biostatistics, Johns Hopkins University, Baltimore, MD 21205, USA}
    \and Shengyi Li\thanks{Department of Biostatistics, Johns Hopkins University, Baltimore, MD 21205, USA}
}
\date{\today}
\maketitle

\begin{abstract}
Large language models (LLMs) are increasingly used as automatic evaluators of generative AI outputs, a paradigm often referred to as ``LLM-as-a-judge.'' In practice, LLM judges are imperfect predictions for the underlying truth and can exhibit systematic, non-random errors. Two main approaches have recently been proposed to address this issue: (i) \emph{direct measurement-error correction} based on misclassification models such as Rogan--Gladen-style estimators, and (ii) \emph{surrogate-outcome approaches} such as prediction-powered inference (PPI), which correct bias by calibrating prediction residuals on a small set of gold-standard human labels. In this paper, we systematically study the performance of these two approaches for estimating \emph{mean parameters} (e.g., average benchmark scores or pairwise win rates). Leveraging tools from semiparametric efficiency theory, we unify the two classes of estimators by deriving explicit forms of \emph{efficient influence function} (EIF)-based efficient estimators and characterize conditions under which PPI-style estimators attain strictly smaller asymptotic variance than measurement-error corrections. We verify our theoretical results in simulations and demonstrate the methods on a real-data example. We provide an implementation of the benchmarked methods and comparison utilities at \texttt{https://github.com/yiqunchen/debias-llm-as-a-judge}.

\end{abstract}

\section{Introduction}
\label{sec:intro}

Automated model evaluation is an essential yet challenging task in machine learning (ML), especially in the era of generative artificial intelligence (AI), where scalability produces large volumes of outputs to be labeled and introduces diverse output formats ranging from numerical predictions and categorical labels to open-ended paragraphs such as reasoning traces and mathematical proofs, all of which are not adequately captured by traditional metrics such as mean squared error or misclassification rates. Moreover, an increasing number of studies now leverage large language models (LLMs) to generate both substantial quantities of new questions and their corresponding answers, for which running every question-answer pair through human verification is infeasible.

Motivated by recent advances in LLMs, a growing body of work proposes using LLMs as automatic ``judges'' to evaluate the outputs of other models---for example, by determining whether a model's answer is correct, assigning a quality score, or selecting the preferred response between two candidates, a practice known as ``LLM-as-a-judge''. This approach promises scalable and low-cost evaluation, but it also introduces a new source of error: the LLM judge is an imperfect proxy for human judgment. 
At a high level, LLM-as-a-judge workflows typically rely on:
\begin{itemize}
  \item a large \emph{evaluation set} on which \emph{only} LLM judgments are available, and
  \item a smaller \emph{calibration set} on which both human labels and LLM judgments are observed.
\end{itemize}
This setup presents a challenge: relying solely on the calibration set leads to high-variance evaluation, whereas relying on the evaluation set labels leads to biased evaluation, as LLMs are known to produce homogenized judgments and to avoid extreme categories. Recent work has therefore asked the key question: \emph{How can we optimally combine these two datasets to construct valid, efficient estimates with calibrated uncertainties (e.g., confidence intervals) for evaluation metrics?}

Statistically, LLM-as-a-judge evaluation can be framed as a classical measurement-error problem. Consider the binary-outcome setting where the human label $Y$ represents the latent truth and the LLM judgment $\hat{Y}$ is a noisy surrogate. We observe paired data $(Y_i, \hat{Y}_i)$ for a small calibration subset, while for most of the evaluation set only $\hat{Y}_i$ is available. A long-standing approach in this literature is to model the misclassification mechanism $\Pr(\hat{Y}\mid Y)$ and estimate the true accuracy using \emph{direct measurement-error correction}, such as Rogan--Gladen estimators or confusion-matrix inversion~\citep{rogan1978estimating, fuller2009measurement}.  Another school of thought, including \emph{prediction-powered inference} (PPI)~\citep{angelopoulos2023prediction} and related work~\citep{wang2020methods, miao2025assumption, salerno2025ipd,ji2025predictionssurrogatesrevisitingsurrogate,egami2023using,rister2025correcting}, treats the LLM judgment as a generic surrogate outcome for $Y$ and uses the calibration set to correct the average bias of $\hat{Y}$. This approach yields valid statistical inference without explicitly estimating the parameters of the misclassification model. 

Despite progress in both methodological traditions, there has been little direct comparison of their statistical validity and efficiency. This gap likely reflects the distinct cultures and communities in which these approaches are typically applied: measurement-error models have recently been used primarily in medical settings, such as estimating COVID-19 and cause-of-death prevalence under imperfect measurements~\citep{rosin2023estimating,gonzalez2017review,fiksel2022generalized}, whereas PPI-style estimators are more commonly motivated by modern ML/AI applications. A recent exception is \citet{lee2025correctly}, which provides a recipe for applying measurement-error corrections to LLM-as-a-judge settings (see details in Section~\ref{subsec:RG}).

While we primarily focus on guarantees of inferential validity for mean parameters in this work, there is a growing body of excellent complementary research aimed at improving LLM-as-a-judge scores ($\hat{Y}$) so that $\hat{Y} \approx Y$ across a diverse set of tasks and LLM models~\citep{tian2025overconfidence,liu2024calibrating,shi2024judging,sahoo2025quantitative,schroeder2024can,li2025calibraeval}. We view this line of work as complementary to ours: our goal is to provide guarantees regardless of the quality of $\hat{Y}$, while improved calibration of $\hat{Y}$ can further reduce variance across the estimators we consider (see Section~\ref{sec:method} for details).

In this paper, we study the performance of these two approaches for estimating \emph{mean parameters} and unify them through the lens of semiparametric efficiency theory (see Figure~\ref{fig:intro} for an overview): From elementary efficient influence function (EIF) theory for mean estimation with a surrogate label $\widehat{Y}$, the optimal estimator depends on the regression function $\mathbb{E}[Y \mid \widehat{Y}]$. This perspective reveals that popular calibration-based estimators can be understood as approximations to $\mathbb{E}[Y \mid \widehat{Y}]$ using different functional forms: PPI implicitly uses the identity map $Y \approx \widehat{Y}$; \texttt{PPI++} uses a one-parameter linear map $Y \approx \lambda \widehat{Y}$ with a tuning parameter $\lambda$; and \texttt{RePPI} replaces these parametric forms with a flexible machine learning model to estimate $\mathbb{E}[Y \mid \widehat{Y}]$ directly~\citep{ji2025predictionssurrogatesrevisitingsurrogate}. In particular, in the binary-outcome setting, \texttt{PPI++} with the optimal $\lambda$ coincides with the EIF-based estimator, whereas for more general outcomes this coincidence need not hold. These sets of insights we derived are closely related to the rich literature on efficient inference for missing data, especially the missing completely at random setup~\citep{Chen_2008,robins1995semiparametric}. Recently, \citet{xu2025unified} proposed a general ``safe PPI'' framework, which improves upon PPI/\texttt{PPI++} when the parameter of interest has dimension $p>1$, and reached a similar conclusion that the asymptotic variance of \texttt{PPI++} can be strictly improved, unless $\mathbb{E}[Y \mid \widehat{Y}]$ takes a restrictive linear form. Our work can be seen as a parallel exploration and implementation for the scalar outcome case, which covers the most common applications, such as LLM-as-a-judge.

Empirically, we find that across a wide range of parameter settings, our proposed EIF-based estimator outperforms vanilla PPI, and that both methods outperform the Rogan--Gladen estimator in terms of finite-sample bias and the efficiency of their confidence intervals. We further explain these empirical findings \emph{theoretically} by proving that: 
(1) the EIF estimator is asymptotically equivalent to \texttt{PPI++} for binary outcomes with the optimal tuning parameter; and 
(2) the asymptotic variance of the PPI estimator is strictly smaller than that of the Rogan--Gladen estimator. 

Our work helps connect the growing literature on the estimation efficiency theory in classic statistics with the rapidly expanding practice of LLM-as-a-judge and calibration in applied AI work. As LLMs increasingly become integral to data analysis workflows, our work offers practical and much-needed guidance for the reliable uncertainty quantification of LLM-as-a-judge systems.

\begin{figure}
    \centering
    \includegraphics[width=\linewidth]{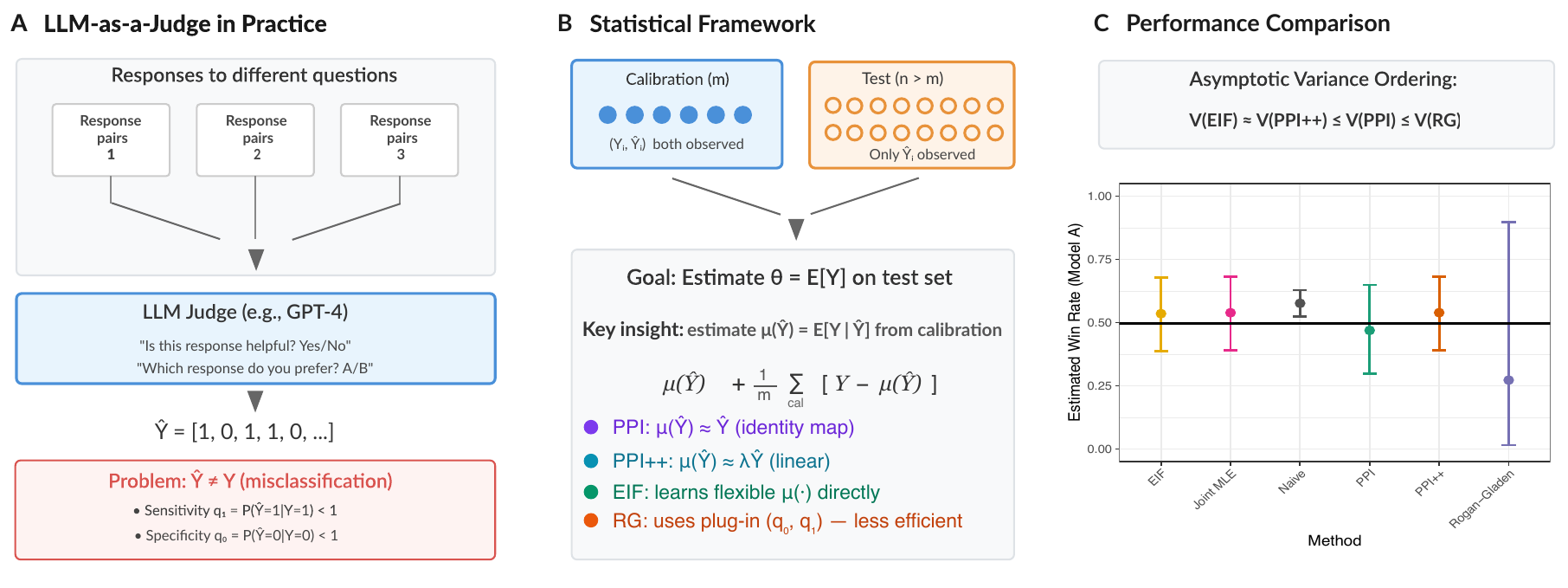}
\caption{\textbf{Calibrating LLM-as-a-judge evaluations.}
\textbf{(A)} LLM-generated labels $\hat{Y}$ are noisy, with sensitivity $q_1<1$ and specificity $q_0<1$.
\textbf{(B)} A human-labeled calibration set enables bias correction using PPI, Rogan--Gladen (RG), or EIF-based estimators.
\textbf{(C)} Comparison of naive, RG, PPI, and EIF-based estimators for evaluating LLM-as-a-judge performance.}
\label{fig:intro}
\end{figure}

\section{Review of Calibration Methods in the Binary Setting}
\label{sec:method}

\subsection{Problem setup}

We stylize the LLM-as-a-judge evaluation as follows: for $i = 1,\dots, N$, we observe i.i.d.\ pairs
\begin{equation}
Y_i \sim \mathrm{Bern}(\theta)
\quad\text{denoting the human label,}
\qquad
\hat Y_i \in \{0,1\}
\quad\text{denoting the LLM-as-a-judge label.}
\label{eq:model-gen}
\end{equation}
The target parameter is the \emph{mean human outcome}
\begin{equation}
\theta = \E[Y] = \Pr(Y=1).
\label{eq:theta-def}
\end{equation}
As a concrete example, if the goal is to compare two responses $R_1$ and $R_2$ produced by different LLMs, let $Y_i=1$ indicate that a human annotator prefers $R_1$ over $R_2$ for prompt $i$ (0 otherwise). The estimand $\theta$ then represents the \emph{human win rate} of $R_1$ over $R_2$.

The full sample of size $N$ is partitioned into two \emph{independent}, disjoint subsets: a \textbf{test set} of size $n$, in which only the surrogate outcome $\hat Y_i$ is observed, and a \textbf{calibration set} of size $m = N - n$, in which both the true outcome $Y_j$ and the surrogate $\hat Y_j$ are observed.
This setup corresponds exactly to a \emph{missing completely at random} (MCAR) mechanism, where the true outcome $Y_i$ is missing for observations in the test set~\citep{robins1995semiparametric,little2002statistical}, because the partition into calibration versus test sets is generated by an independent Bernoulli assignment that does not depend on either $Y_i$ or $\hat Y_i$. 
For asymptotic analysis, we consider a regime in which both $m$ and $n$ diverge with
$\frac{n}{m} \;\to\; \gamma_1 \in (0,\infty).$ That is, the proportion of observations in the calibration set converges to a fixed constant $\frac{1}{1+\gamma_1} \in (0,1)$.

Denote the empirical proportion of LLM-judged positives in the test set by
\begin{equation}
\hat\theta_{\mathrm{naive}}
\;=\;
\frac{1}{n}\sum_{i=1}^n \widehat Y_i .
\label{eq:theta-naive}
\end{equation}
When the judge is imperfect (i.e., $q_0+q_1<2$), we have that $\E[\hat\theta_{\mathrm{naive}}]=\theta+(q_0 + q_1 - 2)\theta + (1 - q_0)$, which leads to biased point estimates, incorrect variance estimates, and invalid inference. We therefore review several calibration-based/debiasing estimators: the Rogan--Gladen (RG) estimator~\citep{rogan1978estimating,lee2025correctly}, prediction-powered inference (PPI)~\citep{angelopoulos2023prediction}, and its power-tuned variant (\texttt{PPI++})~\citep{angelopoulos2023ppipp}.

\subsection{Rogan--Gladen estimator}
\label{subsec:RG}

The Rogan--Gladen estimator accounts for the sensitivity and specificity of the LLM judge,
\begin{equation}
q_1 = \Pr(\hat Y = 1 \mid Y = 1),
\qquad
q_0 = \Pr(\hat Y = 0 \mid Y = 0).
\label{eq:q-param}
\end{equation}
Let $p = \Pr(\hat Y = 1)$ denote the marginal probability that the LLM calls an instance positive. The classical misclassification relationship is
\begin{equation}
p = q_1 \theta + (1-q_0)(1-\theta)
  = (q_0 + q_1 - 1)\theta + (1-q_0).
\label{eq:misclass-eq}
\end{equation}
Solving \eqref{eq:misclass-eq} for $\theta$ yields %
$\theta = \frac{p + q_0 - 1}{q_0 + q_1 - 1}.$
In practice, $p$, $q_0$, and $q_1$ are unknown and estimated from the data: $
\hat p = \frac{1}{n}\sum_{i=1}^n \hat Y_i$.
On the calibration set, write $m_1 = \sum_{j=1}^m \mathbf{1}\{Y_j=1\}$ and $m_0 = m - m_1$ and define
\[
\hat q_1 = 
\frac{\sum_{j=1}^m \mathbf{1}\{\hat Y_j=1,Y_j=1\}}{m_1},
\qquad
\hat q_0 = 
\frac{\sum_{j=1}^m \mathbf{1}\{\hat Y_j=0,Y_j=0\}}{m_0}.
\]
The Rogan--Gladen estimator for $\theta$ is
\begin{equation}
\hat\theta_{\mathrm{RG}} = \frac{\hat p + \hat q_0 - 1}{\hat q_0 + \hat q_1 - 1}.
\label{eq:theta-RG}
\end{equation}

This is the classic prevalence estimator under measurement error~\citep{rogan1978estimating,lang2014confidence} and the same estimator proposed in \citet{lee2025correctly}.

\begin{proposition}[Asymptotic distribution of $\hat\theta_{\mathrm{RG}}$ in \eqref{eq:theta-RG}]
\label{prop:var-RG}
Assume the test and calibration samples are independent and  $q_0+q_1-1 \neq 0$ with $q_0$ and $q_1$ defined in \eqref{eq:q-param}. With $N=n+m \to \infty$ and $n/m\to \gamma_1$, we have that
\[
\sqrt{N}\,\bigl(\hat\theta_{\mathrm{RG}} - \theta\bigr)
\;\xrightarrow{d}\;
\mathcal{N}\bigl(0, V_{\mathrm{RG}}\bigr),
\]
where
\begin{equation}
V_{\mathrm{RG}}
=
\frac{1+\gamma_1}{\gamma_1}\frac{1}{(q_0+q_1-1)^2}
\left[
p(1-p)
\;+\;
\gamma_1\left\{
(1-\theta)\,q_0(1-q_0)
+
\theta\,q_1(1-q_1)
\right\}
\right].
\label{eq:RG-var}
\end{equation}
\end{proposition}

\subsection{Prediction-powered inference (PPI)}
\label{subsec:PPI}

Prediction-powered inference treats the LLM label as a surrogate for $Y$ and corrects its bias using the calibration set without explicitly modeling the sensitivity and specificity. Following the setup in Section~\ref{subsec:RG}, we have that $\E[\hat Y] = p,\,
\E[\hat Y - Y] = \Delta = p-\theta$, motivating the following estimator:
\begin{equation}
\hat\theta_{\mathrm{PPI}} = \hat p - \frac{1}{m}\sum_{j=1}^m (\hat Y_j - Y_j).
\label{eq:theta-PPI}
\end{equation}

\begin{proposition}
\label{prop:PPI-bias-var}
Under the joint model \eqref{eq:model-gen} with independent test and calibration samples:
\begin{enumerate}[label=(\roman*),leftmargin=1.5em]
\item PPI is unbiased for all $n,m$:
\[
\E[\hat\theta_{\mathrm{PPI}}] = \theta.
\]

\item The finite-sample variance is
\begin{equation}
\Var(\hat\theta_{\mathrm{PPI}})
=
\frac{p(1-p)}{n}
+
\frac{(1-\theta)(1-q_0)+\theta(1-q_1)-(\theta-p)^2}{m},
\label{eq:ppi-var}
\end{equation}
where $p=\Pr(\hat Y=1)=(1-\theta)(1-q_0)+\theta q_1$.
If $N=n+m\to\infty$ with $n/m\to\gamma_1\in(0,\infty)$,
\[
\sqrt{N}\,(\hat\theta_{\mathrm{PPI}}-\theta)
\xrightarrow{d}
\mathcal{N}(0,V_{\mathrm{PPI}}),\;V_{\mathrm{PPI}}
=
\left(\frac{1+\gamma_1}{\gamma_1}\right)\left\{p(1-p)
+
\gamma_1\Bigl[(1-\theta)(1-q_0)+\theta(1-q_1)-(\theta-p)^2\Bigr]\right\}.
\]
\end{enumerate}
\end{proposition}

\subsection{Power-tuned PPI (\texttt{PPI++})}
\label{subsec:PPIplus}
\texttt{PPI++} generalizes \eqref{eq:theta-PPI} by introducing a tuning parameter $\lambda\in\mathbb{R}$:
\begin{equation}
\hat\theta_{\mathrm{\texttt{PPI++}}}(\lambda)
:=
\frac{1}{m}\sum_{j=1}^m Y_j
\;+\;
\lambda\left(
\frac{1}{n}\sum_{i=1}^n \hat Y_i
\;-\;
\frac{1}{m}\sum_{j=1}^m \hat Y_j
\right).
\label{eq:theta-PPIplus}
\end{equation}
It's easy to see that \eqref{eq:theta-PPIplus} remains unbiased for any choice of $\lambda$ that is independent of the $Y_i$'s, but different choices lead to different asymptotic variances (and therefore different power), as detailed in Proposition~\ref{prop:PPIplus-bias-var}.

\begin{proposition}
\label{prop:PPIplus-bias-var}
Assume $\lambda$ is fixed (non-data-dependent), under the joint model \eqref{eq:model-gen} with independent test and calibration samples. Then:
\begin{enumerate}[label=(\roman*),leftmargin=1.5em]
\item \emph{(Unbiasedness)} For all $n,m$,
\[
\E\!\left[\hat\theta_{\mathrm{\texttt{PPI++}}}(\lambda)\right] = \theta.
\]

\item \emph{(Finite-sample variance)}
\begin{equation}
\Var\!\left(\hat\theta_{\mathrm{\texttt{PPI++}}}(\lambda)\right)
=
\frac{\theta(1-\theta)}{m}
+
\lambda^2\left(\frac{p(1-p)}{n}+\frac{p(1-p)}{m}\right)
-
\frac{2\lambda}{m}\,\theta\bigl(q_1-p\bigr).
\label{eq:ppiplus-var-finite}
\end{equation}

\item \emph{(Asymptotic distribution)}
If $N=n+m\to\infty$ with $n/m\to\gamma_1\in(0,\infty)$, then
\[
\sqrt{N}\left(\hat\theta_{\mathrm{\texttt{PPI++}}}(\lambda)-\theta\right)
\xrightarrow{d}
\mathcal N\!\left(0,\,V_{\mathrm{\texttt{PPI++}}}(\lambda)\right),
\]
where
\begin{equation}
V_{\mathrm{\texttt{PPI++}}}(\lambda)
=
\frac{1+\gamma_1}{\gamma_1}\left[\gamma_1\,\theta(1-\theta)
+
\lambda^2(1+\gamma_1)\,p(1-p)
-
2\lambda\,\gamma_1\,\theta\bigl(q_1-p\bigr)\right].
\label{eq:ppiplus-var-asymp}
\end{equation}

\item \emph{(Variance-minimizing tuning)}
$V_{\mathrm{\texttt{PPI++}}}(\lambda)$ is minimized over $\lambda\in\mathbb{R}$ at
\[
\lambda^\star
=
\frac{\gamma_1\,\theta\bigl(q_1-p\bigr)}{(1+\gamma_1)\,p(1-p)}.
\]
\end{enumerate}
\end{proposition}

\subsection{Maximum Likelihood Estimator}
\label{subsec:full-mle}

In addition to (moment-based) debiasing estimators such as PPI/\texttt{PPI++}, another natural
candidate is the \emph{joint maximum likelihood estimator} under the full three-parameter
misclassification model for $(Y,\hat Y)$. In contrast to RG (which uses plug-in estimators for sensitivity and specificity on the calibration sample), the MLE explicitly fits $(\theta,q_0,q_1)$ by maximizing the joint
likelihood implied by the model across both the test and calibration samples.

Under model \eqref{eq:model-gen}, we have that the test observations satisfy $\hat Y_i\sim\mathrm{Bern}(p)$,
while the calibration observations satisfy
\[
\Pr(Y=1,\hat Y=1)=\theta q_1,\qquad
\Pr(Y=1,\hat Y=0)=\theta(1-q_1),
\]
\[
\Pr(Y=0,\hat Y=1)=(1-\theta)(1-q_0),\qquad
\Pr(Y=0,\hat Y=0)=(1-\theta)q_0.
\] 

\paragraph{Likelihood and estimating equations.}
Let $O_i=(R_i,Y_i,\hat Y_i)$ for $i=1,\dots,n+m$, where $R_i\in\{0,1\}$ indicates whether $Y_i$ is observed
($R_i=1$ for calibration and $R_i=0$ for test). In this setup we assume $R_i$'s are simply i.i.d. draws with 
$\Pr(R=1)=m/(n+m)$, and the observed-data likelihood factorizes into a calibration part (for $R=1$) and a
test-only part (for $R=0$). Writing $p=(1-\theta)(1-q_0)+\theta q_1$, the per-observation contribution is
\[
f(O_i;\theta,q_0,q_1)
=
\Bigl[\Pr(Y_i,\hat Y_i)\Bigr]^{R_i}\Bigl[\Pr(\hat Y_i)\Bigr]^{1-R_i},
\]
so the full log-likelihood (up to constants) is
\begin{align}
\ell(\theta,q_0,q_1)
&=
\sum_{i=1}^{n+m}(1-R_i)\Bigl\{\hat Y_i\log p+(1-\hat Y_i)\log(1-p)\Bigr\} \nonumber\\
&\quad+
\sum_{i=1}^{n+m}R_i\Bigl\{
\mathbf 1\{Y_i=1,\hat Y_i=1\}\log(\theta q_1)
+\mathbf 1\{Y_i=1,\hat Y_i=0\}\log(\theta(1-q_1)) \nonumber\\
&\qquad
+\mathbf 1\{Y_i=0,\hat Y_i=1\}\log((1-\theta)(1-q_0))
+\mathbf 1\{Y_i=0,\hat Y_i=0\}\log((1-\theta)q_0)
\Bigr\}.
\label{eq:mle-loglik-R}
\end{align}
The MLE is defined as
\[
(\hat\theta_{\mathrm{MLE}},\hat q_{0,\mathrm{MLE}},\hat q_{1,\mathrm{MLE}})
=
\arg\max_{\theta,q_0,q_1\in(0,1)}\ \ell(\theta,q_0,q_1),
\]
equivalently as a solution to the score equations
\[
\frac{\partial}{\partial\theta}\ell(\theta,q_0,q_1)=0,\qquad
\frac{\partial}{\partial q_0}\ell(\theta,q_0,q_1)=0,\qquad
\frac{\partial}{\partial q_1}\ell(\theta,q_0,q_1)=0.
\]

In practice, these equations are solved numerically (e.g., by Newton--Raphson), subject to the
constraint $q_0+q_1\neq 1$ (identifiability).

\begin{proposition}[Asymptotic normality and efficiency of the MLE]
\label{prop:mle-asymp}
Assume the three-parameter misclassification model for $(Y,\hat Y, R)$ is correctly specified and standard
regularity conditions hold. Suppose $N=n+m\to\infty$ with $n/m\to\gamma_1\in(0,\infty)$, and define
\[
p=(1-\theta)(1-q_0)+\theta q_1,
\qquad
\eta=(\theta,q_0,q_1)^\top.
\]
Denote the MLE based on the combined test ($n$) and calibration ($m$) samples as
\[
\hat\eta_{\mathrm{MLE}}
:=
(\hat\theta_{\mathrm{MLE}},\hat q_{0,\mathrm{MLE}},\hat q_{1,\mathrm{MLE}})^\top.
\]
Then
\[
\sqrt{N}\,(\hat\eta_{\mathrm{MLE}}-\eta)\xrightarrow{d}\mathcal N\!\left(0,\ I_{\gamma_1}(\eta)^{-1}\right),
\]
where the limiting Fisher information is
\[
I_{\gamma_1}(\theta,q_0,q_1)
=
\frac{\gamma_1}{1+\gamma_1}\,I_{\mathrm{test}}(\theta,q_0,q_1)
+
\frac{1}{1+\gamma_1}\,I_{\mathrm{cal}}(\theta,q_0,q_1),
\]
with
\[
I_{\mathrm{test}}(\theta,q_0,q_1)
=
\frac{1}{p(1-p)}
\begin{pmatrix}
(q_0+q_1-1)^2 & -(q_0+q_1-1)(1-\theta) & \theta(q_0+q_1-1)\\
-(q_0+q_1-1)(1-\theta) & (1-\theta)^2 & -\theta(1-\theta)\\
\theta(q_0+q_1-1) & -\theta(1-\theta) & \theta^2
\end{pmatrix},
\]
\[
I_{\mathrm{cal}}(\theta,q_0,q_1)
=
\mathrm{diag}\!\left(
\frac{1}{\theta(1-\theta)},\ 
\frac{1-\theta}{q_0(1-q_0)},\ 
\frac{\theta}{q_1(1-q_1)}
\right).
\]
Consequently,
\[
\sqrt{N}\,(\hat\theta_{\mathrm{MLE}}-\theta)\xrightarrow{d}\mathcal N\!\left(0,\ \bigl[I_{\gamma_1}(\eta)^{-1}\bigr]_{11}\right), \; \text{where}
\]
\[
\bigl[I_{\gamma_1}(\eta)^{-1}\bigr]_{11}
=
(1+\gamma_1)\,\theta(1-\theta)\,
\frac{\,p(1-p)+\gamma_1\Bigl[(1-\theta)q_0(1-q_0)+\theta q_1(1-q_1)\Bigr]\;}
{\,p(1-p)+\gamma_1\Bigl[(q_0+q_1-1)^2\theta(1-\theta)+(1-\theta)q_0(1-q_0)+\theta q_1(1-q_1)\Bigr]\;}.
\]
\end{proposition}

\section{A principled efficiency lens via efficient influence functions}
\label{sec:eif}

We now have multiple asymptotically unbiased estimators for $\theta=\E[Y]$ (RG, PPI,
\texttt{PPI++}, and the MLE). This motivates the natural question:
\begin{quote}
    \emph{What is the smallest achievable asymptotic variance under a given statistical model, and which
estimators attain it?}
\end{quote}

We answer this question using semiparametric efficiency theory in statistics (see~\citet{bickel1993efficient,van2000asymptotic} and references therein for a detailed treatment). Let
$Z_1,\dots,Z_N$ denote i.i.d.\ observations from a distribution $P$ and let $\theta=\theta(P)$ be the
target parameter. An estimator $\hat\theta$ is \emph{regular and asymptotically linear} if there exists
a mean-zero function $\phi_P(Z)$ such that $
\hat\theta-\theta(P)
=
\frac{1}{N}\sum_{i=1}^N \phi_P(Z_i) + o_p(N^{-1/2})$.
Any such $\phi_P$ is called an \emph{influence function} (IF) for $\hat\theta$, and a direct application of the central limit theorem yields that
\[
\sqrt{N}\bigl(\hat\theta-\theta(P)\bigr)\ \xrightarrow{d}\ 
\mathcal N\!\left(0,\ \Var_P\!\bigl(\phi_P(Z)\bigr)\right).
\]

Among all IFs, the \emph{efficient influence function} (EIF) for $\theta$ under $P$---when it exists---is the unique influence function with the smallest asymptotic variance. This minimal variance is referred to as the \emph{semiparametric efficiency bound}. An estimator whose influence function coincides with the EIF attains this bound and is therefore said to be \emph{semiparametrically efficient}.

Specifically, let's consider the general model for the joint distribution of $(Y,\hat Y, R)\sim P$, where
$Y\in\{0,1\}$ is the human label; $\hat Y\in\{0,1\}$ is the LLM-judged label; and $R \in \{0,1\}$ indicates whether an observation is in the calibration (observed $Y$) or the test (missing $Y$) set, with no additional constraints beyond the i.i.d. assumption across observations and $R$ is independent of $Y$ and $\hat{Y}$. Recall that we observe $n$ test points
with $\hat Y$ only and $m$ calibration points with both $(Y,\hat Y)$. Let $R\in\{0,1\}$ indicate whether
$Y$ is observed ($R=1$ for calibration and $R=0$ for test with $\Pr(R=1)=\frac{m}{m+n}$), and define
\[
\mu(\hat y):=\E[Y\mid \hat Y=\hat y],\qquad \hat y\in\{0,1\}.
\]

A useful way to motivate the EIF is to observe that there are two natural sources of information about $\theta=\E[Y]$ in this design:
(i) the large test sample provides a precise estimate of the distribution of $\hat Y$; and
(ii) the calibration sample tells us how $Y$ relates to $\hat Y$.
The function $\mu(\hat Y)$ is exactly the \emph{best} way to ``translate'' $\hat Y$ into an estimate of $Y$:
by the law of iterated expectations, $\theta=\E[Y]=\E\{\E[Y\mid \hat Y]\}=\E[\mu(\hat Y)]$.
Thus a natural ``impute-then-average'' estimator is to average $\mu(\hat Y)$ over the (large) test set.
However, since $\mu(\cdot)$ is unknown, we need to learn it from the calibration sample.

\begin{proposition}[Efficient influence function]
\label{prop:eif-modelA}
Given $(Y,\hat Y, R)\sim P$ with $R$ independent of $Y$ and $\hat Y$, the efficient influence function for $\theta=\E[Y]$ under $P$ is
\begin{equation}
\phi^\star(R,Y,\hat Y; P)
=
\mu(\hat Y)-\theta
+\,\frac{R}{\Pr(R=1)}\bigl\{Y-\mu(\hat Y)\bigr\}.
\label{eq:eif-modelA}
\end{equation}
\end{proposition}

\begin{proposition}
\label{prop:eif-estimator}
Define
\[
\hat\mu(1):=\frac{\sum_{j=1}^m \mathbf 1\{\hat Y_j=1\}\,Y_j}{\sum_{j=1}^m \mathbf 1\{\hat Y_j=1\}},\;
\hat\mu(0):=\frac{\sum_{j=1}^m \mathbf 1\{\hat Y_j=0\}\,Y_j}{\sum_{j=1}^m \mathbf 1\{\hat Y_j=0\}},\;
\hat\mu(\hat Y)=\hat\mu(0)\mathbf 1\{\hat Y=0\}+\hat\mu(1)\mathbf 1\{\hat Y=1\}.
\]
Then an asymptotically efficient estimator for estimating $\theta$ under \eqref{eq:model-gen} can be written as
\begin{equation}
\hat\theta_{\mathrm{EIF}}
=
\frac{1}{N}\sum_{i=1}^{N}\hat\mu(\hat Y_i)
\;+\;
\frac{1}{m}\sum_{j=1}^m\Bigl(Y_j-\hat\mu(\hat Y_j)\Bigr),
\label{eq:eif-est}
\end{equation}
with the asymptotic distribution 
\begin{equation}
\sqrt{N}\bigl(\hat\theta_{\mathrm{EIF}} - \theta\bigr) \xrightarrow{d} \mathcal{N}\bigl(0, V_{\mathrm{EIF}}\bigr),
    \end{equation}
  where the variance in the binary case simplifies to
    \begin{equation}
        \label{eq:eif-asymp}
  V_{\mathrm{EIF}} = 
  \frac{\theta(1-\theta)}{p(1-p)} \left[\theta(1-\theta)(q_0+q_1-1)^2 + {(\gamma_1+1)}\bigl(q_1(1-q_1)\theta + q_0(1-q_0)(1-\theta)\bigr)\right].
  \end{equation}
\end{proposition}
\paragraph{\textbf{PPI is not generally efficient.}}
General efficiency theory implies that the EIF estimator $\hat\theta_{\mathrm{EIF}}$ in
\eqref{eq:eif-est} is asymptotically efficient under the general i.i.d.\ model for
$(Y,\widehat Y,R) \sim P$. Consequently, its asymptotic variance in
\eqref{eq:eif-asymp} is the smallest among all regular, asymptotically linear
estimators in this model. Moreover, with some algebra, one can verify that all
estimators introduced in Section~\ref{sec:method} are regular and asymptotically
linear. In particular, PPI corresponds
to the special choice $\mu(\hat Y)= \hat Y$  in \eqref{eq:eif-modelA}. This yields an unbiased estimator, but it is not
generally efficient unless $\E[Y\mid \hat Y]=\hat Y$, in which case $\hat Y$ is a perfect prediction for $Y$.

\paragraph{\textbf{Optimally tuned \texttt{PPI++} matches the EIF in the binary case.}}
Unlike PPI, \texttt{PPI++} searches over a one-parameter family of unbiased estimators and chooses
$\lambda$ to minimize asymptotic variance. In general, \texttt{PPI++} does not coincide with the EIF estimator,
because $\E[Y\mid \hat Y]$ need not equal $\lambda \hat Y$ for a single scalar $\lambda$. However, when
$Y,\hat Y\in\{0,1\}$, the regression function $\mu(\hat Y)=\E[Y\mid \hat Y]$ is \emph{always affine} in
$\hat Y$, so the \texttt{PPI++} family is rich enough to match the EIF estimator. Indeed, letting $\gamma_1=n/m$, the
variance-minimizing choice from Proposition~\ref{prop:PPIplus-bias-var} can be written as
\[
\lambda^\star
=
\frac{\gamma_1\,\Cov(Y,\hat Y)}{(1+\gamma_1)\Var(\hat Y)}
=
\frac{\gamma_1\,\theta(1-\theta)\,(q_0+q_1-1)}{(1+\gamma_1)\,p(1-p)},
\qquad
p=(1-\theta)(1-q_0)+\theta q_1.
\]
On the other hand, by Bayes' rule,
\begin{equation}
\mu(1)=\Pr(Y=1\mid \hat Y=1)=\frac{\theta q_1}{p},
\qquad
\mu(0)=\Pr(Y=1\mid \hat Y=0)=\frac{\theta(1-q_1)}{1-p}.
\end{equation}
Therefore,
\begin{align*}
\mu(1)-\mu(0)
&=
\frac{\theta q_1}{p}-\frac{\theta(1-q_1)}{1-p}
=\theta\,\frac{q_1-p}{p(1-p)}\\
&=
\theta\,\frac{q_1-\bigl[(1-\theta)(1-q_0)+\theta q_1\bigr]}{p(1-p)}
=
\frac{\theta(1-\theta)(q_0+q_1-1)}{p(1-p)}.
\end{align*}
Thus, $\lambda^\star=\frac{\gamma_1}{1+\gamma_1}\bigl\{\mu(1)-\mu(0)\bigr\}$, which is exactly the slope for the affine form of $\mu(\hat Y)$, showing that optimally tuned \texttt{PPI++} is
asymptotically equivalent to the EIF estimator in the binary setting.

\paragraph{\textbf{MLE is equivalent to EIF}}
Under the correctly specified model $(Y, \widehat Y, R) \sim P$, classical likelihood
theory implies that the MLE is asymptotically efficient for $\theta$
(Proposition~\ref{prop:mle-asymp}) and therefore attains the smallest possible
asymptotic variance~\citep{bickel1993efficient}. In the binary-outcome case, the
model is saturated by the three parameters $(\theta, q_0, q_1)$, and thus we
expect the asymptotic variance of the EIF estimator to coincide with that of the
MLE. The detailed algebra establishing the equivalence of the variance expressions
is deferred to the appendix.

\paragraph{RG is less efficient than PPI.} Finally, it is instructive to compare \eqref{eq:theta-RG} and \eqref{eq:theta-PPI}, which yields the following result in Proposition~\ref{prop:PPI-dominates-RG}.

\begin{proposition}
\label{prop:PPI-dominates-RG}
Fix $\theta\in(0,1)$ and $q_0,q_1\in(0,1)$ with $q_0+q_1-1>0$, and assume $n,m\to\infty$ with $n/m\to\gamma_1\in(0,\infty)$. Then
\[
V_{\mathrm{PPI}} \;\le\; V_{\mathrm{RG}},
\]
where
\[
\frac{\gamma_1}{1+\gamma_1}V_{\mathrm{PPI}} = p(1-p)
+
\gamma_1\Bigl[(1-\theta)(1-q_0)+\theta(1-q_1)-(\theta-p)^2\Bigr]
\]
and 
\[
\frac{\gamma_1}{1+\gamma_1}V_{\mathrm{RG}} =\frac{1}{(q_0+q_1-1)^2}
\left[
p(1-p)
\;+\;
\gamma_1\left\{
(1-\theta)\,q_0(1-q_0)
+
\theta\,q_1(1-q_1)
\right\}
\right]
\]
for all such parameters, with equality if and only if $q_0=q_1=1$ (i.e. $\Pr(\hat{Y} = Y)=1$).
\end{proposition}

\paragraph{Asymptotic Efficiency Summary.}
The relationships derived above imply a clear ordering of performance. Under the general i.i.d.\ model for $(Y,\hat Y)$, the EIF estimator $\hat\theta_{\mathrm{EIF}}$ attains the semiparametric efficiency bound. In the binary case, optimally tuned \texttt{PPI++} is asymptotically equivalent to $\hat\theta_{\mathrm{EIF}}$, and consequently to the MLE. Formally, for $N=n+m \to \infty$ and $\frac{n}{m}\to\gamma_1$:
\[
\hat\theta_{\mathrm{EIF}} 
\;=_{\mathrm{asymp}}\; 
\hat\theta_{\mathrm{MLE}} 
\;=_{\mathrm{asymp}}\; 
\hat\theta_{\mathrm{\texttt{PPI++}}}(\lambda^\star) 
\;\prec_{\mathrm{asymp}}\; 
\hat\theta_{\mathrm{PPI}} 
\;\prec_{\mathrm{asymp}}\; 
\hat\theta_{\mathrm{RG}},
\]
where $=_{\mathrm{asymp}}$ denotes equivalence up to $o_p(n^{-1/2})$ and $\prec_{\mathrm{asymp}}$ denotes smaller asymptotic variance.

\section{Beyond binary: Nonparametric Estimation of $\mathbb{E}[Y\mid \hat Y]$}
\label{subsec:eif-beyond}

In the binary case, $\mu(\hat Y) = \mathbb{E}[Y\mid \hat Y]$ is necessarily affine, allowing optimally tuned \texttt{PPI++} to achieve the efficiency bound. Moreover, the nonparametric distribution on $(Y, \hat{Y})$ is equivalent to the three-parameter distribution used for the MLE in the binary case; this ``saturated'' property allows the MLE to attain the efficiency bound \emph{without} making any restrictive parametric assumptions.

For richer surrogate outputs $\widehat Y$ (e.g., continuous scores and ordinal ratings), the functional form of the EIF, and therefore of
$\hat\theta_{\mathrm{EIF}}$ in \eqref{eq:eif-est}, remains unchanged. However, the
regression function $\mu(\cdot)$ may be nonlinear, and efficiency depends on how
close the estimator $\widehat{\mu}(\cdot)$ is to the true conditional mean function.

When both $Y$ and $\hat Y$ are one-dimensional, estimating $\mu(\hat Y)$ reduces to a univariate regression problem, which can be addressed using models ranging from simple linear regression to flexible nonparametric smoothers. Given a consistent estimator $\hat{\mu}(\cdot)$ and i.i.d. model $(Y,\hat Y, R)$, we can then construct the plug-in EIF estimator:
\begin{equation}
\label{eq:continupus-theta}
    \hat\theta_{\mathrm{EIF}} = \frac{1}{N} \sum_{i=1}^{N} \left( \hat{\mu}(\hat Y_i) + \frac{1}{\hat{P}(R=1)} R_i \{Y_i - \hat{\mu}(\hat Y_i)\} \right).
\end{equation}
If we \emph{design} the labeling process $R$, e.g., we  label a data point with probability $P(R=1)=\frac{\gamma_1}{1+\gamma_1}$, then as long as 
$\hat{\mu}(\cdot)$ is a consistent estimator of $\mu(\cdot)$, \eqref{eq:continupus-theta} will be asymptotically efficient. 

We finally note that, unlike the binary (or the categorical with finite categories) case, we do not have a direct analogy with the MLE or RG estimator in the continuous case without assuming additional parametric relationship of $Y$ and $\hat{Y}$. 

\section{Simulation Study}
\label{sec:simulation}

\subsection{Binary $Y$}

We compare the performance of following estimators in simulation:
  the naive estimator in~\eqref{eq:theta-naive}, the Rogan--Gladen estimator in~\eqref{eq:theta-RG}, the PPI estimator in~\eqref{eq:theta-PPI}, the \texttt{PPI++} estimator in~\eqref{eq:theta-PPIplus},
  the MLE in \eqref{eq:mle-loglik-R}, and the \texttt{EIF} estimator in~\eqref{eq:eif-est}.

\subsection{Simulation Setup}
\label{subsec:sim-setup}

  \paragraph{Data-generating process.}
  Each replicate consists of $N = 2{,}000$ i.i.d.\ samples.
  The true binary label is generated as 
  $Y_i \stackrel{\text{i.i.d.}}{\sim} \mathrm{Bern}(\theta)$,
  where $\theta \in \{0.1, 0.2, \ldots, 0.9\}$ is the target prevalence.

The LLM-as-a-judge surrogate $\widehat Y_i$ is generated from a misclassification
model with constant sensitivity and specificity,
$\Pr(\widehat Y_i = 1 \mid Y_i = 1) = q_1$ and
$\Pr(\widehat Y_i = 0 \mid Y_i = 0) = q_0$, respectively.

  \paragraph{Parameter grid.}
  We vary the following parameters across a full factorial design:
  \begin{itemize}
      \item True prevalence: $\theta \in \{0.1, 0.2, \ldots, 0.9\}$.
      \item Judge accuracy: $q_0, q_1 \in \{0.6, 0.7, 0.8\}$.
      \item Labeling budget (calibration size): $\{1\%, 5\%, 10\%\}$ of $N$
            (i.e., $m \in \{20, 100, 200\}$ human labels in expectation).
  \end{itemize}
  The calibration set is obtained by \emph{random sampling} from the full dataset; thus the class counts
  $m_0$ and $m_1$ in the calibration set are random and depend on the underlying prevalence $\theta$.
  For each configuration, we perform $B = 1{,}000$ Monte Carlo replicates to compute bias as well as coverage and average width of 90\% confidence intervals. 

  \paragraph{Confidence interval construction.}
  We apply a logit transformation to all confidence intervals to improve finite-sample coverage when $\theta$ is near the boundary~\citep{brown2001interval,stone1996course}. Given point estimate $\hat\theta$ and variance estimate $\widehat V$, we form
  \[
  \mathrm{CI}_{1-\alpha} = \mathrm{expit}\!\left(
    \mathrm{logit}(\hat\theta) \pm z_{\alpha/2} \cdot \frac{\sqrt{\widehat V}}{\hat\theta(1 - \hat\theta)}
  \right),
  \]
  where the standard error on the logit scale is obtained via the delta method.  

\paragraph{Evaluation metrics.}
For each estimator $\hat\theta$, let $\hat\theta^{(b)}$ denote its value in replicate 
$b = 1,\dots,B$, and let $\mathrm{CI}^{(b)} = [L^{(b)}, U^{(b)}]$ denote its corresponding 
90\% confidence interval. We compute:
\begin{itemize}
\item \textbf{Bias:}
$
\mathrm{Bias}(\hat\theta)
= \frac{1}{B} \sum_{b=1}^B \bigl(\hat\theta^{(b)} - \theta\bigr);
$
\item \textbf{Coverage probability:}
$\mathrm{Cov}(\hat\theta)
= \frac{1}{B} \sum_{b=1}^B 
   \mathbf{1}\!\left\{\, L^{(b)} \le \theta \le U^{(b)} \,\right\};$
\item \textbf{Mean confidence interval width:}
$\mathrm{Width}(\hat\theta)
= \frac{1}{B} \sum_{b=1}^B \bigl(U^{(b)} - L^{(b)}\bigr).$
\end{itemize}

\subsection{Results under symmetric judge accuracy ($q_0 = q_1$)}
\label{subsec:sim-results-diag}

Figures~\ref{fig:bias-diag}--\ref{fig:ciwidth-diag} are organized with rows corresponding to accuracy levels
$q_0 = q_1 \in \{0.6, 0.7, 0.8\}$ and columns corresponding to labeling budgets.

\paragraph{Bias.}
Figure~\ref{fig:bias-diag} shows that all estimators except the naive one exhibit
smaller average bias than the naive sample-mean estimator. Recall that the bias of
the naive estimator is $(1 - q_0) + \theta (q_0 + q_1 - 2)$, which is negative for
large $\theta > 0.5$ and positive for small $\theta < 0.5$, consistent with our
simulation results. Among the remaining estimators, $\hat\theta_{\mathrm{RG}}$
exhibits comparatively larger finite-sample bias.

\begin{figure}[htbp!]
    \centering
\includegraphics[width=\textwidth]{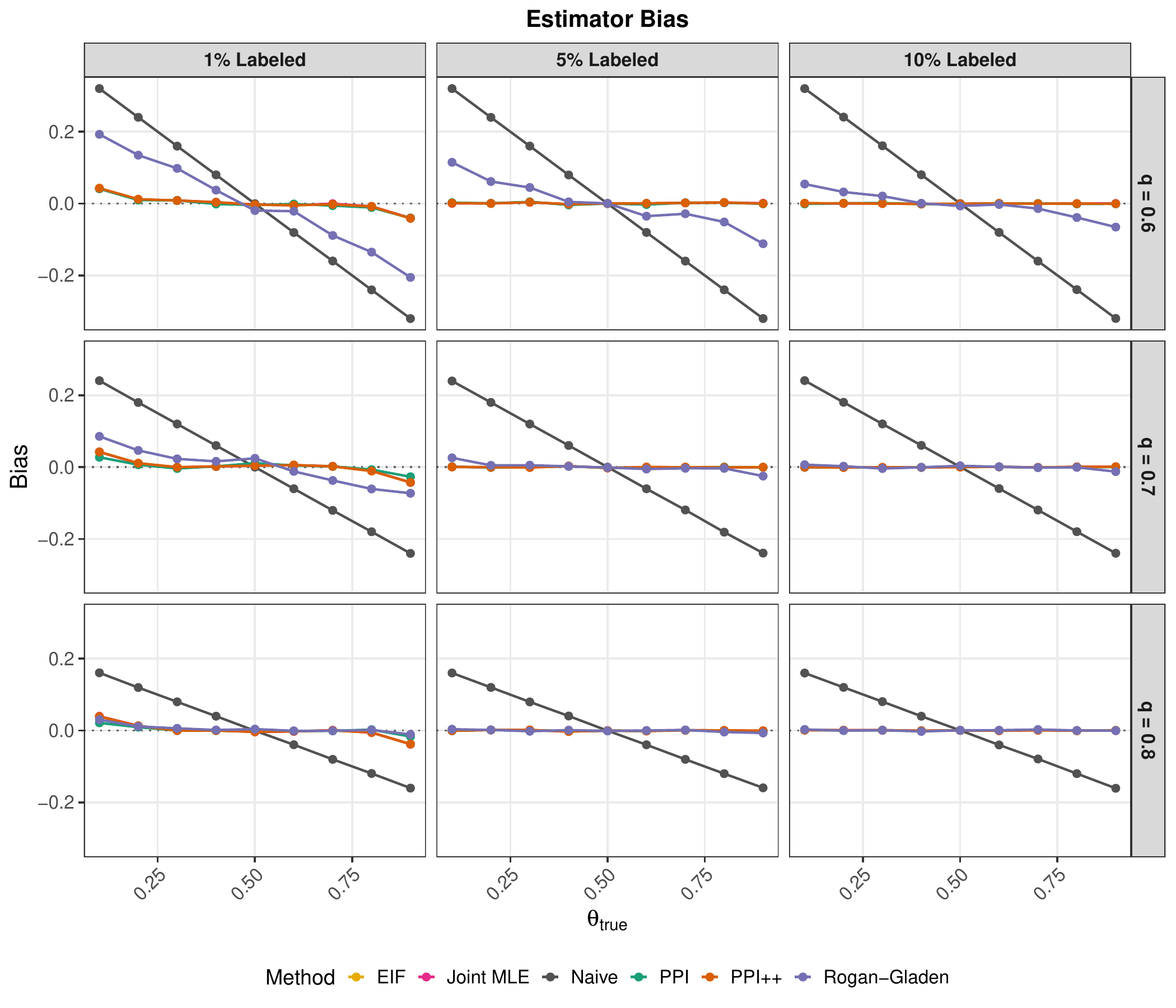}
    \caption{Estimator bias of $\hat\theta$. 
    All debiased estimators achieve near-zero bias; the naive estimator (red) exhibits
    large bias in many settings. When only 1\% of the data is labeled, $\hat\theta_{\mathrm{RG}}$ also exhibits considerable bias.}
    \label{fig:bias-diag}
\end{figure}

\paragraph{Coverage.}
Figure~\ref{fig:coverage-diag} reports empirical coverage rates.
The naive estimator exhibits severe undercoverage, consistent with its asymptotic invalidity (at $\theta=0.5$ the naive estimator is unbiased so the coverage is closer to nominal level). All other estimators maintain near-nominal coverage across settings. The RG estimator and its associated confidence intervals tend to overcover, particularly when the number of labeled observations is small (i.e., when the calibration set is limited).

\begin{figure}[htbp!]
    \centering
\includegraphics[width=0.95\textwidth]{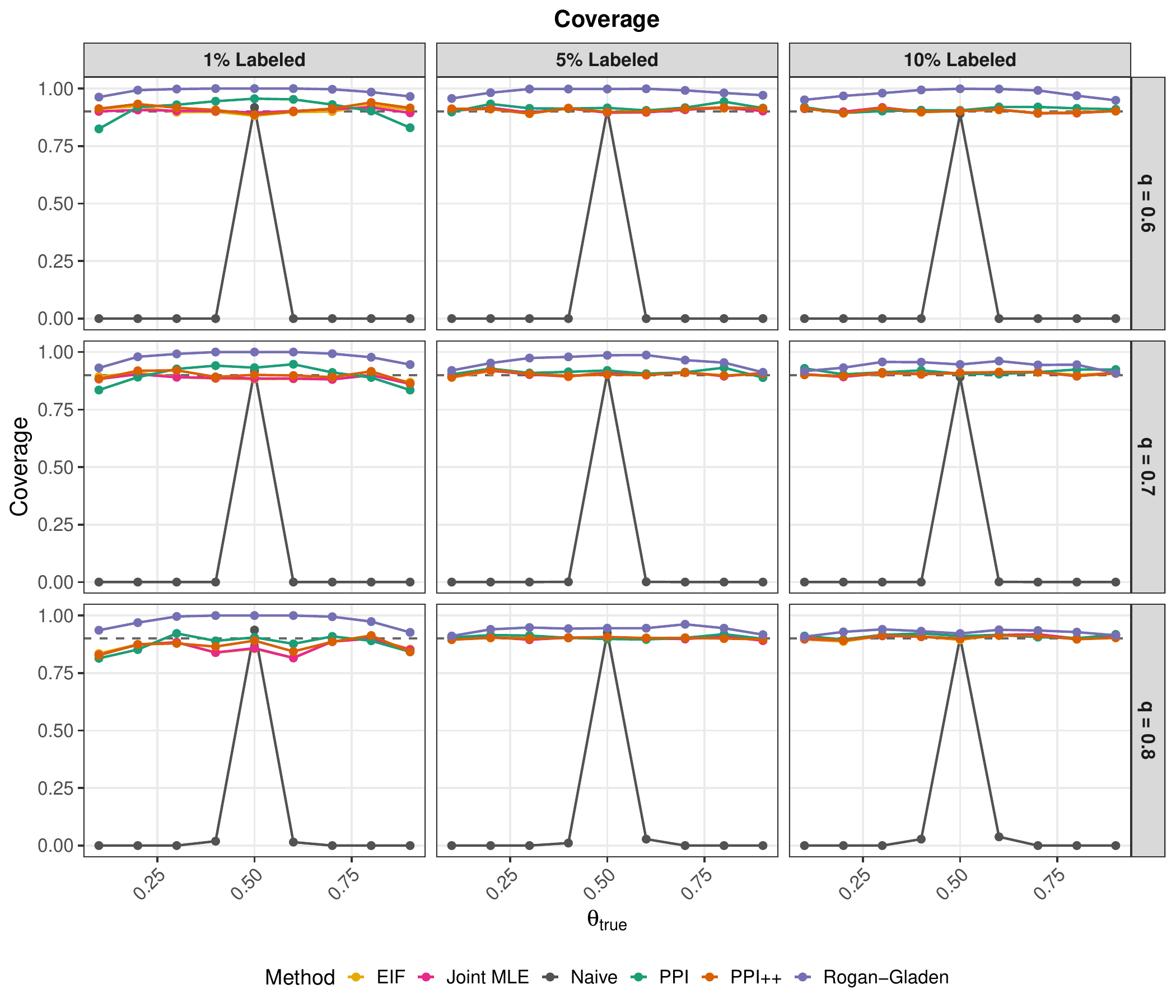}
\caption{Empirical coverage rates under the symmetric case $q_0 = q_1$.
The naive estimator exhibits severe undercoverage. RG, PPI, and \texttt{PPI++} generally achieve nominal coverage; however, RG tends to overcover when the labeled fraction $m/n$ is small and judge quality is low (i.e., smaller $q_0 + q_1$).}
    \label{fig:coverage-diag}
\end{figure}

  \paragraph{Confidence interval width.}
  Figure~\ref{fig:ciwidth-diag} compares interval widths among valid estimators (i.e., $\hat\theta_{\mathrm{naive}}$ excluded). As our theory predicted in Section~\ref{sec:eif}, \texttt{EIF} and \texttt{PPI++} produce nearly identical and consistently shortest intervals,
  typically 35--55\% narrower than standard \texttt{PPI}. 
  \texttt{MLE} intervals are slightly wider but overall comparable to \texttt{EIF} and \texttt{PPI++}. Rogan--Gladen intervals widen dramatically under low-accuracy judges due to
  the $1/(q_0+q_1-1)^2$ amplification factor in its asymptotic variance---at $q_0 = q_1 = 0.6$,
  RG intervals are roughly 10$\times$ wider than those of \texttt{EIF}/\texttt{PPI++}.
  \begin{figure}[t]
      \centering
      \includegraphics[width=0.95\textwidth]{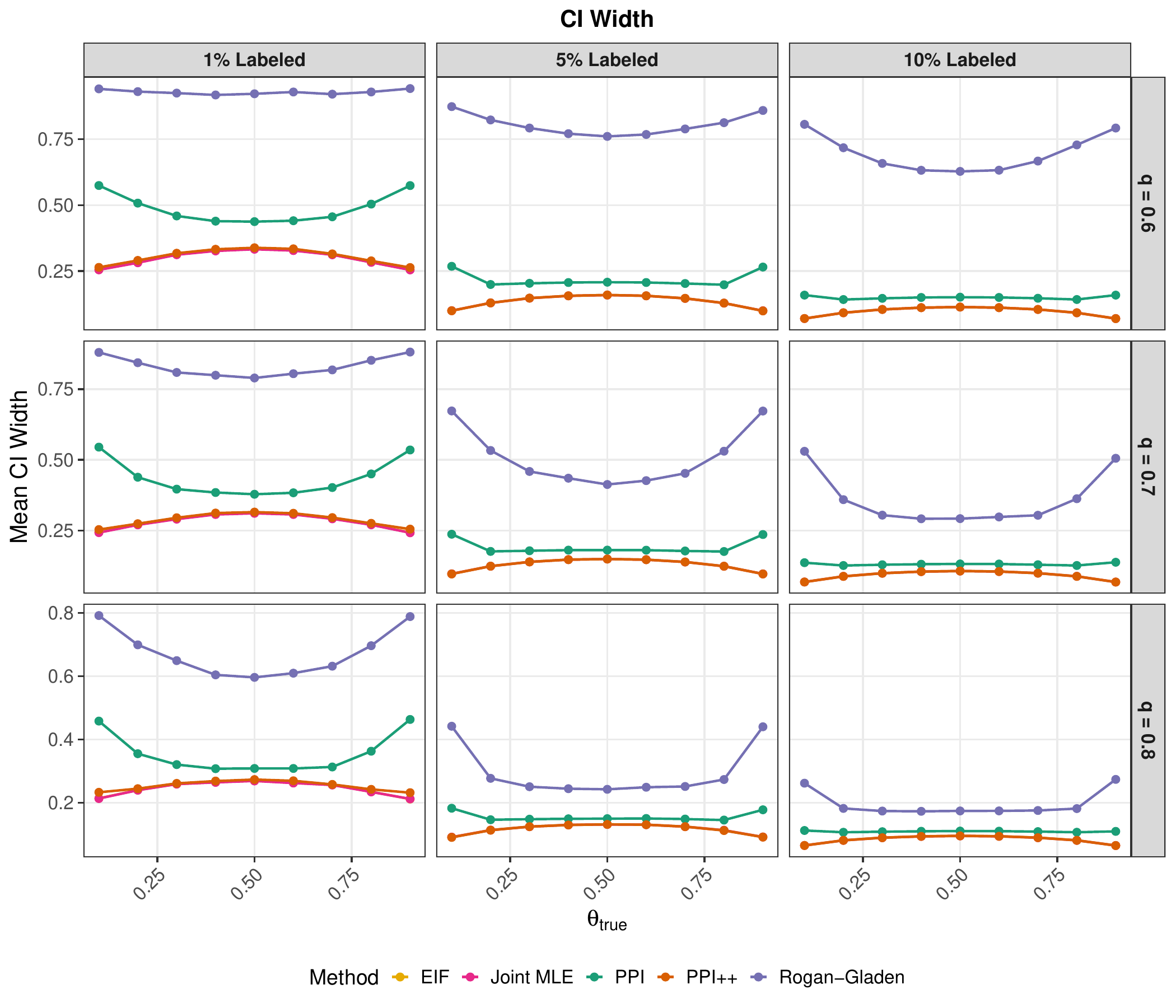}
      \caption{Mean confidence interval width for the symmetric case $q_0 = q_1$.
      \texttt{EIF} and \texttt{PPI++} produce nearly identical and shortest intervals,
      outperforming Rogan--Gladen by a factor of 3--15$\times$ depending on the labeling ratio.
      The advantage is most pronounced when $q_0 + q_1 - 1$ is small (i.e., when the LLM-judge is closer to random guessing).
      \texttt{MLE} produces slightly wider intervals but achieves coverage closer to nominal. Naive confidence intervals are excluded since they are too narrow and do not achieve the desired coverage.}
      \label{fig:ciwidth-diag}
  \end{figure}

\subsection{Continuous $Y$}
\label{subsec:sim-continuous-Y}
Here, we consider the simulation setup in \citet{ji2025predictionssurrogatesrevisitingsurrogate} where $\hat{Y}$ takes discrete values (e.g., LLM ratings on a 1--5 scale) while the true outcome $Y$ is continuous. This scenario is common when LLMs are used as judges, providing ordinal ratings, but the ultimate goal is to infer the continuous quality scale for the models $Y$.

\paragraph{Data Generating Process.}
Let $Z \sim \text{Unif}(\{1, 2, 3\})$ denote a latent mixture component, with conditional distribution $Y \mid Z \sim \cN(\mu_Z, \sigma^2)$ where $\bmu = (\mu_1, \mu_2, \mu_3)^\top$. The surrogate prediction is on the component $Z$, and it's easy to see that we need to carefully model $\E[Y|\hY]$. Even when the prediction $\hY$ perfectly identifies the latent class, the naive estimator in \eqref{eq:theta-naive} is biased for $\E[Y]$ because $\hY$ and $Y$ are on different scales. Specifically, the bias would be $2 - (\mu_1 + \mu_2 + \mu_3)/3$, which grows as the $\mu_k$ values diverge from their category indices.

\paragraph{Simulation Design.}
We fix $N = 2000$, $\mu_1 = 1$, $\mu_2 = 2$, $\sigma = 1$, and vary $\mu_3 \in \{3, 4, \ldots, 9\}$ to control bias magnitude. The labeled fraction varies over $\{5\%, 10\%, 20\%\}$. We report coverage of 90\% confidence intervals and RMSE across $B = 500$ replications.

Here, we compare the performance of the naive estimator in~\eqref{eq:theta-naive}, the
  prediction-powered estimator in~\eqref{eq:theta-PPI}, the power-tuned \texttt{PPI++} estimator in~\eqref{eq:theta-PPIplus},
 EIF estimator in~\eqref{eq:continupus-theta}. For the EIF estimator, we use the following three variants to approximate $\E[Y|\hY]$: (1) {EIF}: Learns $g(z) = \E[Y \mid \hY = z]$ separately for each $z \in \{1,2,3\}$; (2) {EIF-linear}: Fits $g(\hY) = a + b \cdot \hY$ via a linear regression (note this is asymptotically equivalent to \texttt{PPI++}); and (3) {EIF-GAM/spline}: Nonparametric calibration treating $\hY$ as smooth function. Since we are \emph{not} imposing additional distributional assumptions on $Y$, we cannot write down a parametric likelihood function in the continuous case, and therefore cannot easily extend the Rogan--Gladen and the MLE estimators to this setting.

\paragraph{Results.}
Figure~\ref{fig:discrete-coverage-width} shows how inferential performance depends on how $\E[Y\mid\hY]$ is estimated. When $\bmu=(1,2,3)$, the conditional mean is linear in $\hY$, so using $\E[Y\mid\hY]=\hY$ is correctly specified and all methods except the naive estimator achieve nominal coverage. Moreover, the confidence-interval width is indistinguishable across methods, as EIF, PPI, and \texttt{PPI++} are asymptotically equivalent. However, as $\mu_3$ increases, the conditional mean is no longer linear. While all debiased estimators retain correct coverage, there is a clear efficiency gradient: PPI is the least efficient, followed closely by \texttt{PPI++} and EIF with a linear model for $\E[Y\mid\hY]$ (these two are asymptotically equivalent under optimal tuning of \texttt{PPI++}). Finally, using a flexible model to estimate $\E[Y\mid\hY]$, such as a GAM~\citep{wood2012mgcv} or spline, or fitting a three-category model that learns $g(z)=\E[Y\mid\hY=z]$ separately for each $z\in\{1,2,3\}$, achieves the most efficient tier.
\begin{figure}[htbp]
    \centering
    \includegraphics[width=\textwidth]{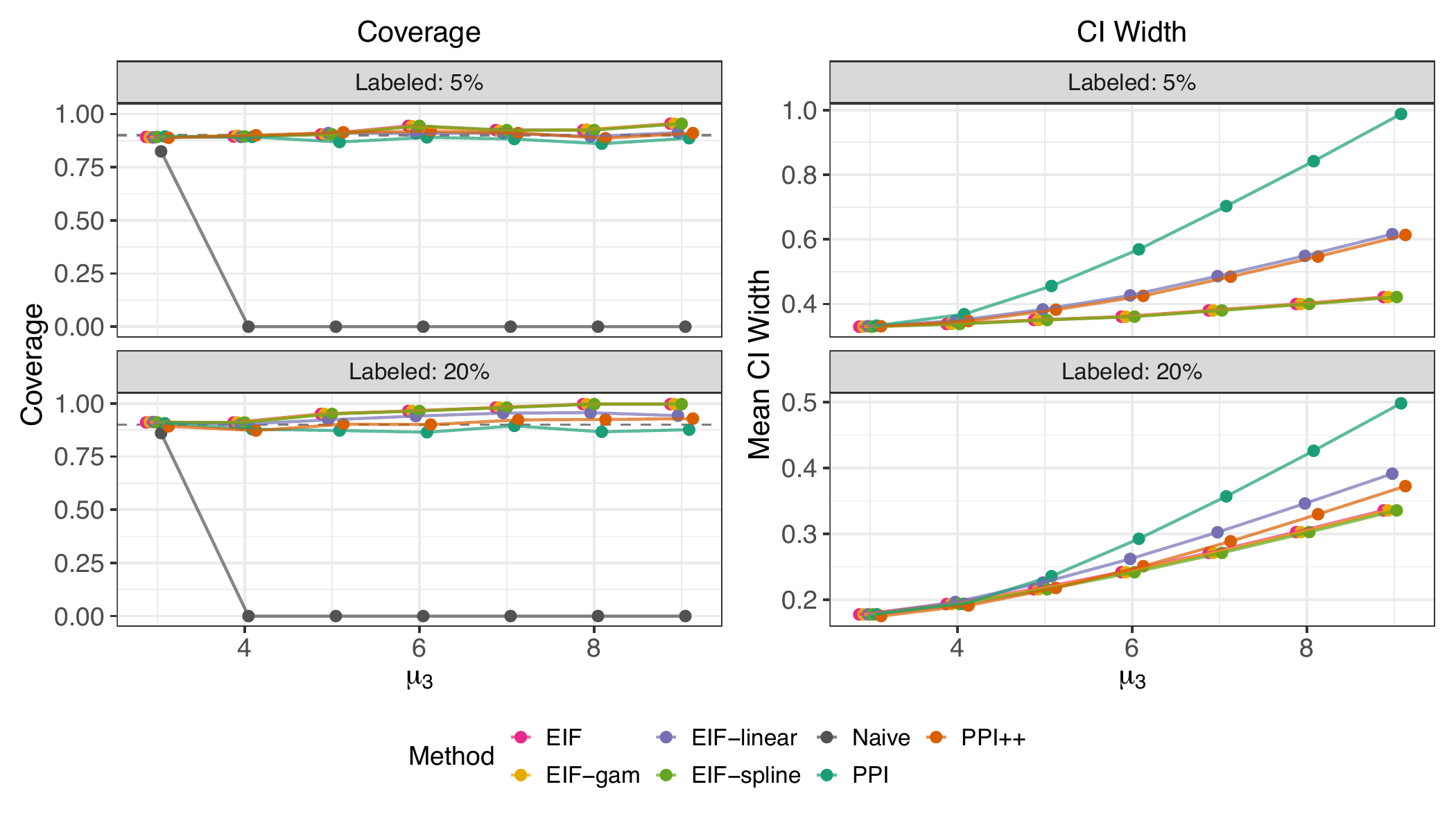}
\caption{Coverage (left) and CI width (right) for discrete prediction simulations. The dashed line indicates nominal 90\% coverage. Most methods achieve approximately valid coverage across varying bias magnitudes ($\mu_3$), with naive intervals severely undercovering as bias increases. EIF with spline and GAM calibration achieves the narrowest intervals, reflecting its ability to correct discrete predictions. As predicted, linear correction and \texttt{PPI++} tread closely.}
  \label{fig:discrete-coverage-width}
\end{figure}

\section{Real Data Application}
  \label{sec:real-data}

  \subsection{Data and Experimental Setup}

  We evaluate the estimators on real human preference data from the \texttt{arena-human-preference-140k} dataset, which is collected on Chatbot Arena~\citep{chiang2024chatbotarenaopenplatform}, an online platform where users compare responses from two AI models and select a winner.
 We focus on three high-frequency model pairs, all with Claude Opus 4 as model A: (1) Claude Opus 4 vs.\ Gemini 2.5 Flash ($n = 493$); (2) Claude Opus 4 vs.\ Gemini 2.5 Pro ($n = 414$); (3) Claude Opus 4 vs.\ Qwen3-235B ($n = 494$).

  For each battle, the ground truth $Y_i = 1$ if the human selected model A (Claude), and $Y_i = 0$ otherwise. We deploy two LLM judges, GPT-4o-mini and GPT-5.2, to produce surrogate labels $\widehat{Y}_i \in \{0, 1\}$ based on the same prompt-response pairs shown to human raters.

  To mimic a real-world setting where we have only access to part of the preference data, we use a 10\% random sample of the data as a labeled calibration set to estimate either $\mathbb{E}[Y \mid \hat{Y}]$ or the misclassification probabilities $(q_0, q_1)$, and the remaining 90\% as an unlabeled test set. All confidence intervals are constructed at the 90\% nominal level.
 \subsection{Results}

\paragraph{Judge quality.}
Table~\ref{tab:judge-quality} reports the estimated specificity ($\hat q_0$) and sensitivity ($\hat q_1$) for each judge--dataset combination, computed from a randomly sampled $10\%$ calibration set.
The condition $q_0 + q_1 > 1$ indicates discriminative power beyond random guessing.

Judge quality varies substantially across settings.
GPT-4o-mini performs best when evaluating Claude~4 vs.\ Gemini~2.5 Flash, but worse on Claude~4 vs.\ Gemini~2.5 Pro.
GPT-5.2 achieves the highest overall discrimination on Claude~4 vs.\ Gemini~2.5 Pro.
Both judges perform poorly on Claude~4 vs.\ Qwen3-235B, exhibiting low specificity despite moderate sensitivity.
These heterogeneous error rates further underscore the need for bias correction: naive estimation inherits the systematic errors of the judge.

  \begin{table}[t]
      \centering
      \caption{Estimated judge quality metrics from calibration data ($\hat{q}_0$: specificity; $\hat{q}_1$: sensitivity). Values of $q_0 + q_1$ near 1 and 2 indicate near-random and near-perfect performance, respectively.}
      \label{tab:judge-quality}
      \begin{tabular}{llcccc}
          \toprule
          Model Pair & Judge & $\hat{q}_0$ & $\hat{q}_1$ & $\hat{q}_0 + \hat{q}_1$ \\
          \midrule
          Claude 4 vs. Gemini 2.5 Flash & GPT-4o-mini & 0.74 & 0.69 & 1.43 \\
           Claude 4 vs. Gemini 2.5 Flash & GPT-5.2     & 0.47 & 0.67 & 1.14 \\
          \midrule
           Claude 4 vs. Gemini 2.5 Pro   & GPT-4o-mini & 0.50 & 0.81 & 1.31 \\
          Claude 4 vs.  Gemini 2.5 Pro   & GPT-5.2     & 0.64 & 0.82 & 1.46 \\
          \midrule
           Claude 4 vs. Qwen3-235B       & GPT-4o-mini & 0.44 & 0.70 & 1.14 \\
           Claude 4 vs. Qwen3-235B       & GPT-5.2     & 0.50 & 0.64 & 1.14 \\
          \bottomrule
      \end{tabular}
  \end{table}

\paragraph{Illustrative examples on one split.}
Figure~\ref{fig:ci-examples} shows point estimates and $90\%$ confidence intervals for a single random calibration/test split across all three model pairs.
The horizontal black line indicates the true win rate, computed from all human labels. We observe that: 
(i) the naive estimator is consistently biased, overestimating the true win rate when using GPT-4o-mini as a judge for Claude~4 Opus;
(ii) the Rogan--Gladen estimator produces extremely wide intervals that are largely uninformative in practice; and
(iii) \texttt{PPI++}, \texttt{EIF}, and the MLE yield similar point estimates with moderate interval widths that cover the true value.

  \begin{figure}[htbp!]
      \centering
      \begin{subfigure}[b]{0.6\textwidth}
          \centering
          \includegraphics[width=\textwidth]{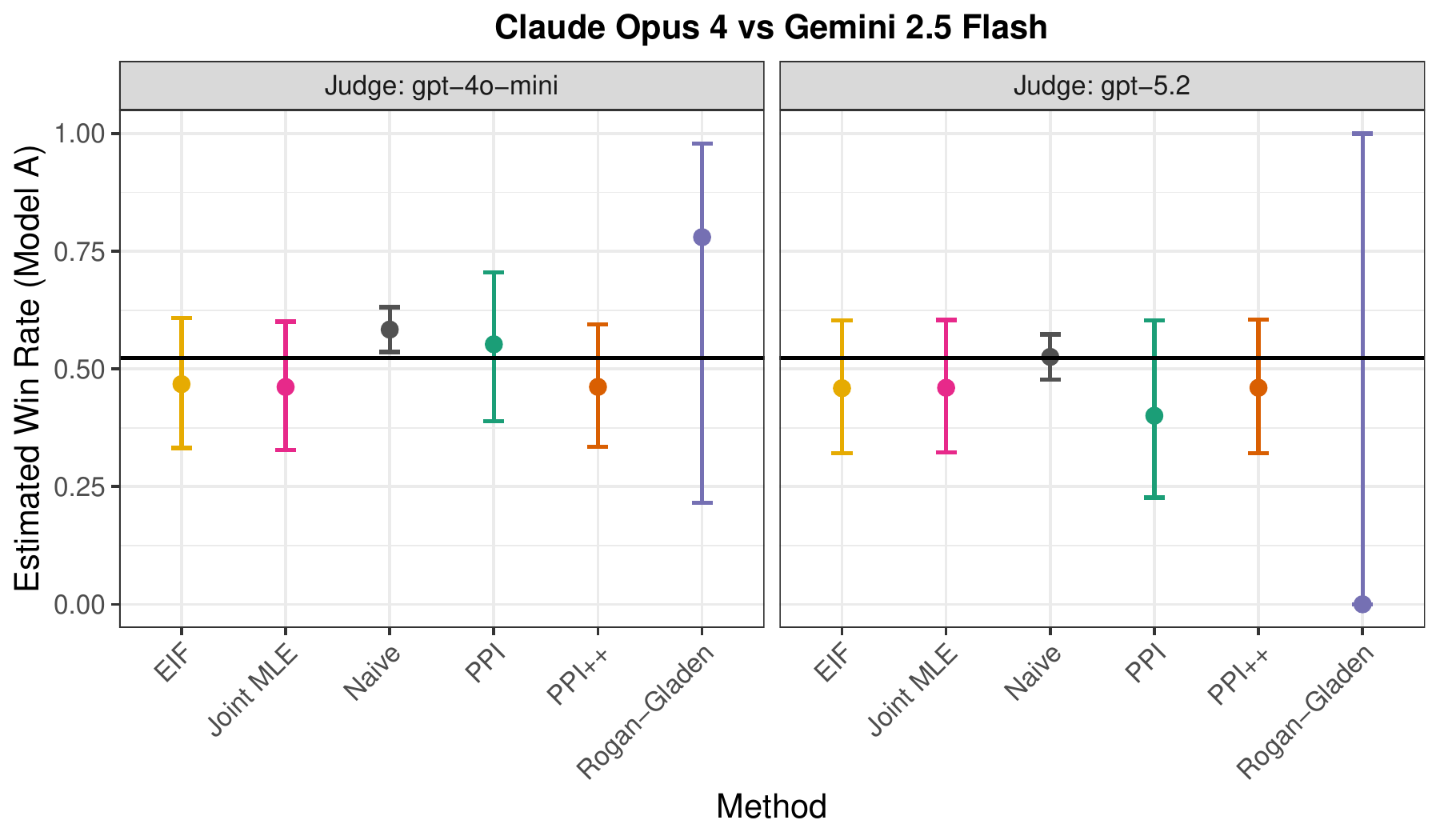}
          \caption{Claude Opus 4 vs.\ Gemini 2.5 Flash (true win rate: 0.523)}
      \end{subfigure}

      \vspace{0.5em}

      \begin{subfigure}[b]{0.6\textwidth}
          \centering
          \includegraphics[width=\textwidth]{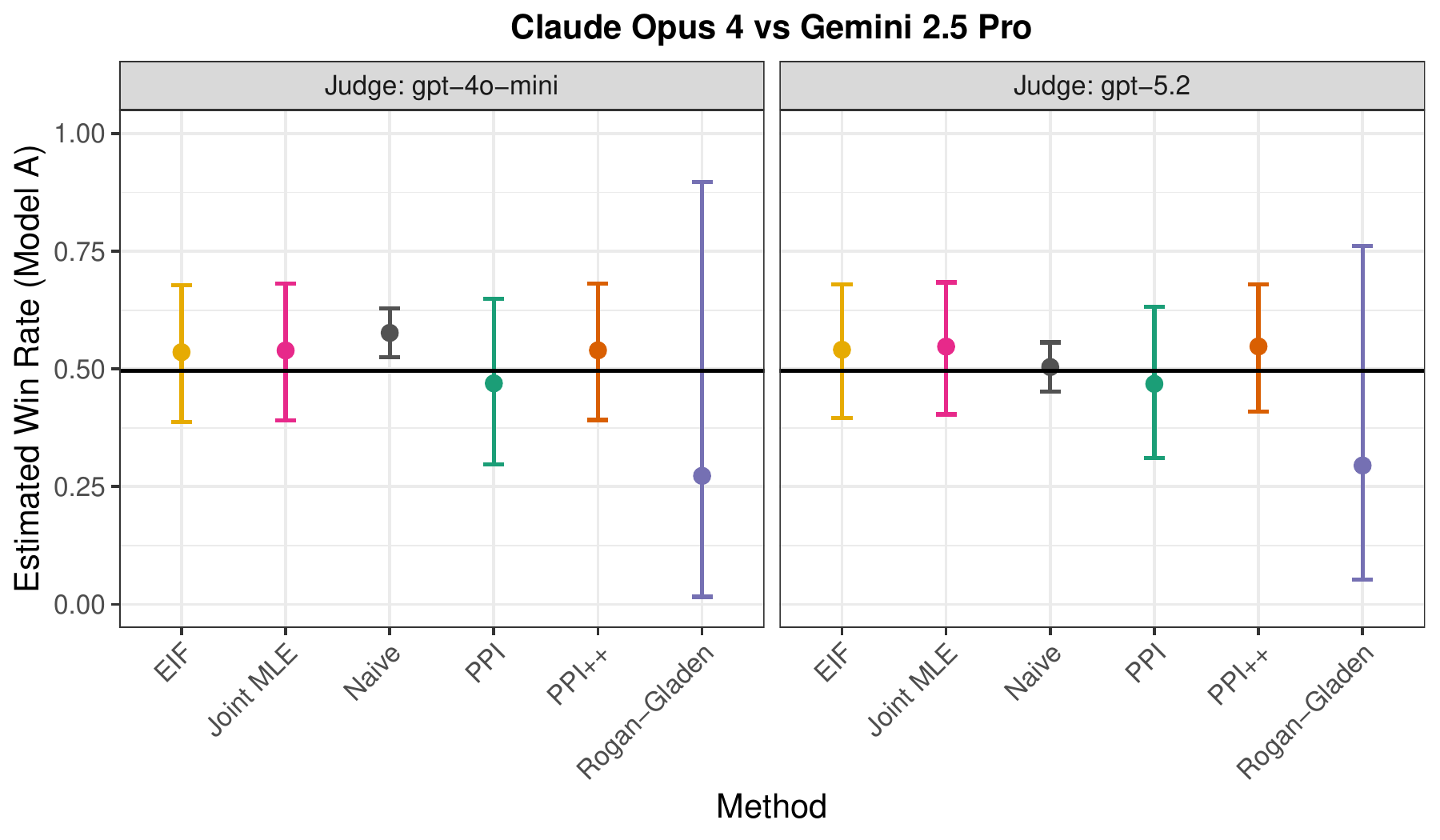}
          \caption{Claude Opus 4 vs.\ Gemini 2.5 Pro (true win rate: 0.496)}
      \end{subfigure}

      \vspace{0.5em}

      \begin{subfigure}[b]{0.6\textwidth}
          \centering
          \includegraphics[width=\textwidth]{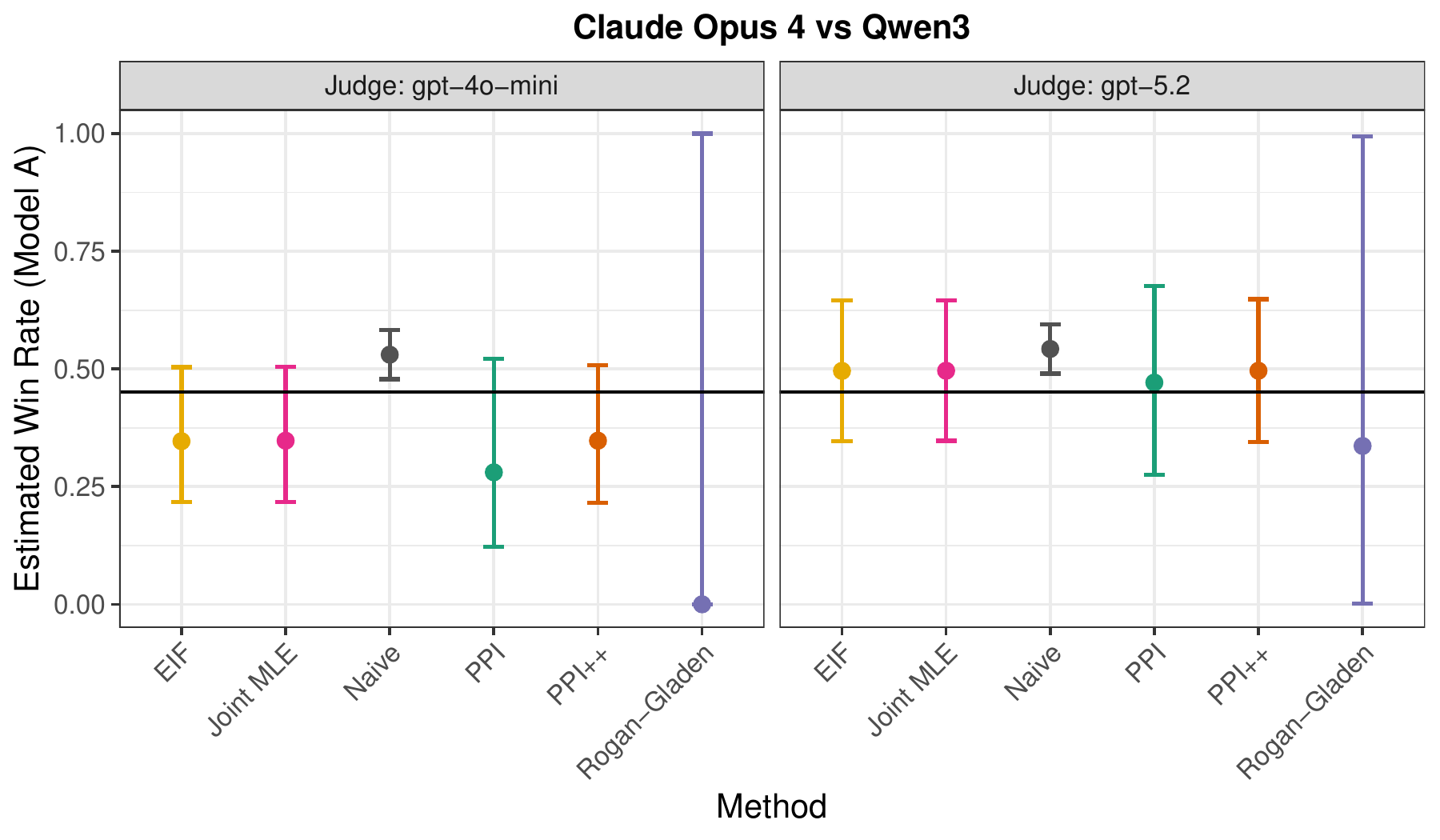}
          \caption{Claude Opus 4 vs.\ Qwen3-235B (true win rate: 0.451)}
      \end{subfigure}

      \caption{Point estimates and 90\% confidence intervals for win rate estimation across three model pairs, using two LLM judges (GPT-4o-mini and GPT-5.2). Horizontal black line shows the true win rate. Different estimators and CIs are displayed in different colors.}
      \label{fig:ci-examples}
  \end{figure}

\paragraph{Coverage and efficiency.}
To assess coverage properties, Figure~\ref{fig:real-data-coverage} plots empirical coverage against mean confidence interval width for each estimator across 1{,}000 random calibration/test splits. 
The naive estimator is unreliable, exhibiting coverage that varies dramatically with the magnitude and direction of judge bias: it attains near-perfect coverage when systematic errors happen to align with the true prevalence (e.g., GPT-5.2 on Gemini~2.5 Pro), but collapses to $0\%$ coverage when bias is substantial, as observed for GPT-4o-mini on Gemini~2.5 Pro and for both judges on Qwen3-235B.

In contrast, \texttt{PPI++}, \texttt{EIF}, and MLE  achieve near-nominal coverage across all judge--dataset combinations, with mean interval widths of $0.27$--$0.30$.
PPI slightly overcovers with wider intervals ($0.35$--$0.44$), while Rogan--Gladen is highly conservative, attaining $100\%$ coverage at the cost of very wide intervals ($0.76$--$0.96$).
Among methods with valid coverage, \texttt{PPI++}, EIF, and the MLE offer the best coverage--width tradeoff, as predicted by our theory, achieving nominal coverage with intervals roughly three times narrower than Rogan--Gladen.

  \begin{figure}[htbp!]
      \centering
      \includegraphics[width=\textwidth]{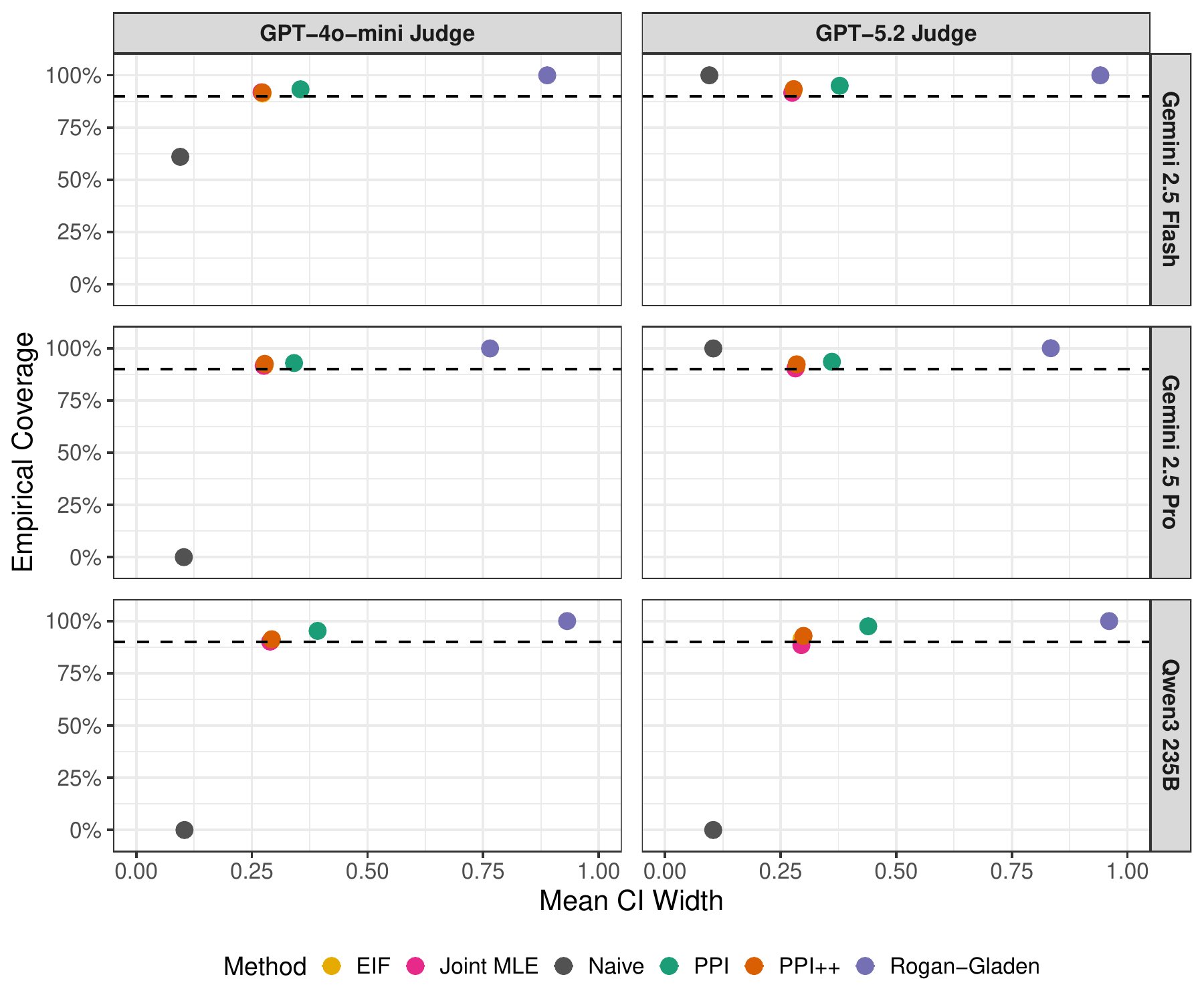}
      \caption{Empirical coverage versus mean CI width across 1{,}000 random 90/10 test/calibration splits. Dashed line indicates 90\% nominal coverage. Error bars show 90\% confidence intervals. Naive estimation (gray) achieves 0\% coverage when judge bias is substantial, while \texttt{PPI++}, EIF, and MLE maintain nominal coverage with narrow intervals.}
      \label{fig:real-data-coverage}
  \end{figure}
  
\section{Discussion}
\label{sec:discussion}

We presented a unified lens for understanding calibration in LLM-as-a-judge evaluations via efficient influence functions.
We compared, theoretically and through simulations, two complementary strategies: direct measurement-error correction via misclassification models and surrogate-based methods such as PPI, which treat LLM judgments as surrogate outcomes.
When the target is a mean outcome---as is common in LLM-as-a-judge settings---both RG and PPI variants provide simple, reliable solutions, with \texttt{PPI++}/EIF-based approaches preferred for their variance advantages.
In particular, for binary outcomes, optimally tuned \texttt{PPI++} is equivalent to the EIF-based strategy.


Several extensions follow naturally.
A key direction is instance-dependent misclassification, where LLM error rates vary with input features or example difficulty.
Extending our framework to this setting parallels the semiparametric efficient inference literature: with
$m(X,\hat Y) = \E[Y \mid X,\hat Y]$, our final estimator involves an influence function of the form
\begin{equation}
\label{eq:eif-instance-dependent}
\phi^\star(O)
= m(X,\hat Y) - \theta
+ \frac{R}{\Pr(R=1 \mid X,\hat Y)}\{Y - m(X,\hat Y)\},
\end{equation}

Another direction is developing multi-judge ensemble methods that combine heterogeneous LLM evaluators to improve robustness.
This would require a careful treatment of multiple judges with distinct error rates $(q_0, q_1)$ in \eqref{eq:q-param}, as well as principled ways to aggregate their surrogate labels $\hat Y$ (e.g., via optimal weighting or learned combination rules).

Finally, robustness to distribution shift, including covariate shift (e.g., systematic differences in response styles between the calibration and test sets) and label shift (e.g., changes in the prevalence of ground-truth labels $Y$ between calibration and deployment, such as when rolling out models to new cultural or geographic contexts), warrants deeper methodological and empirical study, and we leave these theoretical and computational developments to future work.

As LLM-as-a-judge evaluations become increasingly central to model development and scientific assessment, we hope our results provide a practical foundation for principled calibration, valid uncertainty quantification, and efficient use of limited human labels.

\clearpage

\bibliographystyle{plainnat}
\bibliography{references}

\clearpage

\appendix
\begin{center}
    \textbf{\Large Supplementary Materials for Unifying Debiasing Methods for LLM-as-a-Judge Evaluations}
\end{center}

\section{Additional Results}
\subsection{Finite-sample variability in calibration estimates}
  \label{subsec:sim-calibration-bias}
The Rogan--Gladen estimator relies on plug-in estimates $\hat q_0$ and $\hat q_1$.
Figure~\ref{fig:qbias-diag} illustrates their finite-sample variability across simulation settings.
While these estimates are approximately unbiased, their variance can be substantial when the labeling budget is small.
At $1\%$ labeling ($m = 20$), the positive class may contain only $m_1 \approx 3$ observations,
yielding an RMSE of $0.3$ for $\hat q_1$, compared to $0.1$ for $\hat q_0$.

Because the RG estimator divides by $\hat q_0 + \hat q_1 - 1$, the resulting confidence intervals can be wide when $m$ is small
or when class imbalance leaves $m_1$ (or $m_0$) with few observations.

In contrast, PPI-type estimators and EIF-based approaches do not require explicit estimation of $(q_0, q_1)$.
Instead, they exploit unbiased correction to learn $\mathbb{E}[Y \mid \hat Y]$,
making them more stable in small-calibration regimes.

  \begin{figure}[htbp!]
      \centering
      \includegraphics[width=0.95\textwidth]{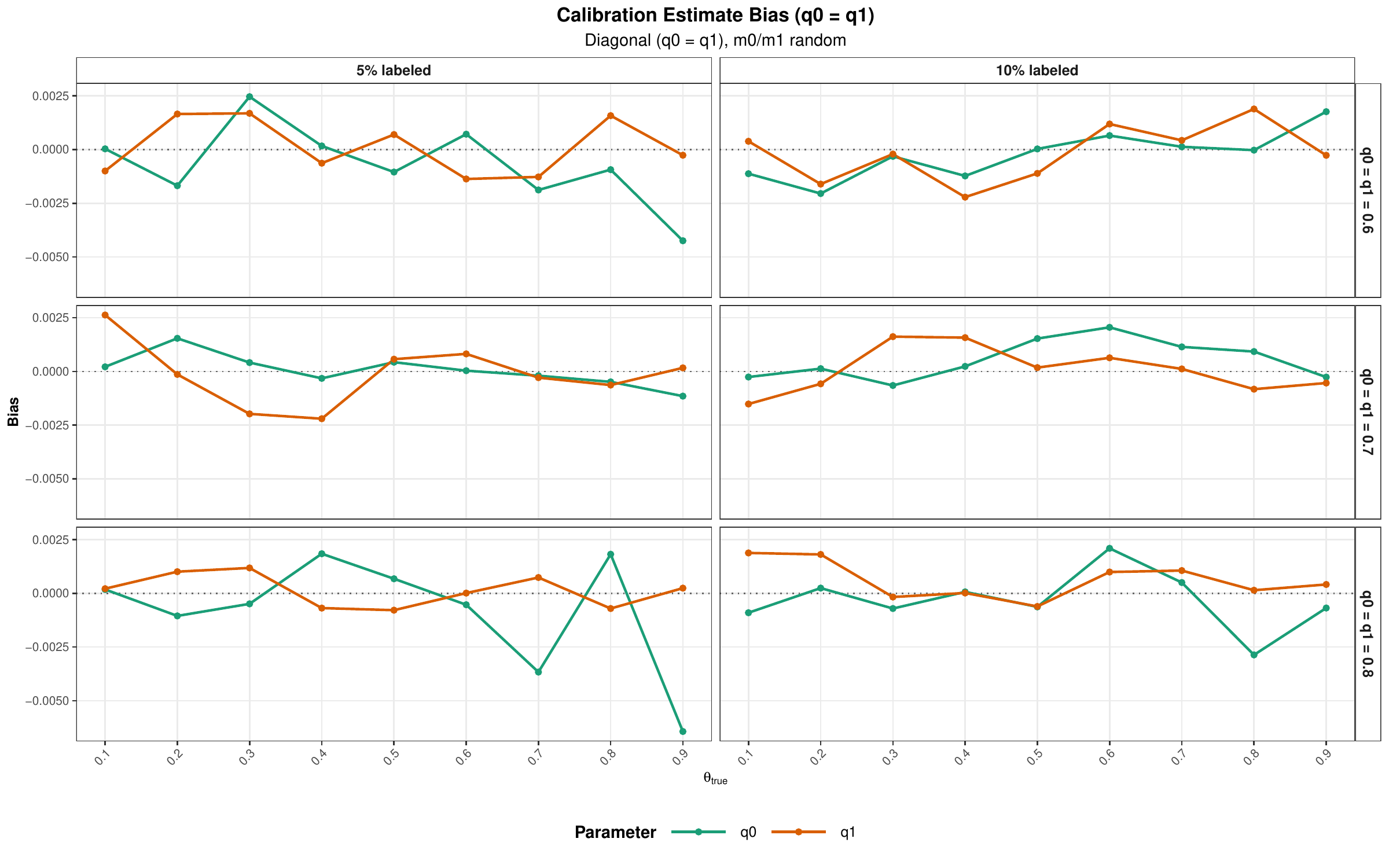}
      \caption{Finite-sample RMSE of calibration estimates $\hat q_0$ and $\hat q_1$.}
      \label{fig:qbias-diag}
  \end{figure}
\section{Proofs of Propositions}
\subsection{Proof of Proposition~\ref{prop:var-RG}}
\begin{proof}
Since the test sample consists of i.i.d.\ draws of $\hat Y\in\{0,1\}$ with $
p = \Pr(\hat Y=1)$,
it follows that $\hat Y_i \sim \mathrm{Bern}(p)$ for $i=1,\dots,n$. For the sample mean $\hat{p}=\frac1n\sum_{i=1}^n \hat Y_i$, as $n\to\infty$,
by the central limit theorem, we have that
\[
\sqrt{n}\,(\hat{p}-p)\;\xrightarrow{d}\;\mathcal{N}\!\bigl(0,p(1-p)\bigr).
\]

\noindent
On the calibration sample, define
\[
X_{ab}= \sum_{j=1}^m\mathbf{1}\{Y_j=a,\ \hat Y_j=b\},
\qquad \pi_{ab}=\Pr(Y=a,\hat Y=b),
\qquad a,b\in \{0,1\}.
\]
Then $X=(X_{00},X_{01},X_{10},X_{11})^\top$ satisfies
\[
X\sim\mathrm{Multinomial}\bigl(m;\,\pi_{00},\pi_{01},\pi_{10},\pi_{11}\bigr),
\]
where
\begin{align*}
    \pi_{00}&=\Pr(Y=0,\hat Y=0)=q_0(1-\theta),\\
    \pi_{01}&=\Pr(Y=0,\hat Y=1)=(1-q_0)(1-\theta),\\
    \pi_{10}&=\Pr(Y=1,\hat Y=0)=(1-q_1)\theta,\\
    \pi_{11}&=\Pr(Y=1,\hat Y=1)=q_1\theta.
\end{align*}
Let $\pi=(\pi_{00},\pi_{01},\pi_{10},\pi_{11})^\top$ and $\hat{P}(R=1)^{\mathrm{cal}}=X/m$. As $m\to\infty$,
by the multivariate central limit theorem,
\[
\sqrt{m}\,(\hat{P}(R=1)^{\mathrm{cal}}-\pi)
\;\xrightarrow{d}\;
\mathcal{N}\!\left(0,\,\mathrm{diag}(\pi)-\pi\pi^\top\right).
\]

\noindent
Since
\[
\hat q_1=\frac{\sum_{j=1}^m\mathbf{1}\{\hat Y_j=1,\ Y_j=1\}}{m_1},
\qquad 
\hat q_0=\frac{\sum_{j=1}^m\mathbf{1}\{\hat Y_j=0,\ Y_j=0\}}{m_0},
\]
it follows that
\[
\hat q_1=\frac{X_{11}}{X_{11}+X_{10}}=\frac{\hat \pi^{\mathrm{cal}}_{11}}{\hat{P}(R=1)^{\mathrm{cal}}_{11}+\hat{P}(R=1)^{\mathrm{cal}}_{10}},
\qquad 
\hat q_0=\frac{X_{00}}{X_{00}+X_{01}}=\frac{\hat{P}(R=1)^{\mathrm{cal}}_{00}}{\hat{P}(R=1)^{\mathrm{cal}}_{00}+\hat{P}(R=1)^{\mathrm{cal}}_{01}}.
\]

\noindent
Define $h:(0,1)^4\to(0,1)^2$ by
\[
h(x_{00},x_{01},x_{10},x_{11})
=
\left(
\frac{x_{00}}{x_{00}+x_{01}},
\ \frac{x_{11}}{x_{10}+x_{11}}
\right)^{\!\top},
\]
so that $h(\pi)=(q_0,q_1)^\top$ and $h(\hat{P}(R=1)^{\mathrm{cal}})=(\hat q_0,\hat q_1)^\top$.
Since $\pi_{00}+\pi_{01}=1-\theta>0$ and $\pi_{11}+\pi_{10}=\theta>0$, the mapping $h$ is continuously
differentiable at $\pi$. Its Jacobian at $\pi$ is
\[
\nabla h(\pi)
=
\begin{pmatrix}
\dfrac{\pi_{01}}{(\pi_{00}+\pi_{01})^2} &
-\dfrac{\pi_{00}}{(\pi_{00}+\pi_{01})^2} &
0 & 0 \\[10pt]
0 & 0 &
-\dfrac{\pi_{11}}{(\pi_{10}+\pi_{11})^2} &
\dfrac{\pi_{10}}{(\pi_{10}+\pi_{11})^2}
\end{pmatrix}.
\]
As $m\to\infty$, by the multivariate delta method,
\[
\sqrt{m}
\begin{pmatrix}
\hat q_0-q_0\\
\hat q_1-q_1
\end{pmatrix}
\xrightarrow{d}
\mathcal{N}\!\left(
0,\,
\nabla h(\pi)\bigl(\mathrm{diag}(\pi)-\pi\pi^\top\bigr)\nabla h(\pi)^\top
\right).
\]
\noindent
We note that $\theta = \pi_{11}+\pi_{10}$; therefore a direct calculation yields
\[
\nabla h(\pi)\bigl(\mathrm{diag}(\pi)-\pi\pi^\top\bigr)\nabla h(\pi)^\top
=
\begin{pmatrix}
\dfrac{q_0(1-q_0)}{1-\theta} & 0 \\
0 & \dfrac{q_1(1-q_1)}{\theta}
\end{pmatrix},
\] and 
\[
\sqrt{m}
\begin{pmatrix}
\hat q_0-q_0\\
\hat q_1-q_1
\end{pmatrix}
\xrightarrow{d}
\mathcal{N}\!\left(
0,\,
\begin{pmatrix}
\dfrac{q_0(1-q_0)}{1-\theta} & 0 \\
0 & \dfrac{q_1(1-q_1)}{\theta}
\end{pmatrix}
\right).
\]

\noindent
Since $n/m\to\gamma_1\in(0,\infty)$, we have $\sqrt{n}=\sqrt{n/m}\,\sqrt{m}$ and
$\sqrt{n/m}\to\sqrt{\gamma_1}$. By Slutsky's theorem,
\[
\sqrt{n}
\begin{pmatrix}
\hat q_0-q_0\\
\hat q_1-q_1
\end{pmatrix}
\xrightarrow{d}
\mathcal{N}\!\left(
0,\,
\gamma_1
\begin{pmatrix}
\dfrac{q_0(1-q_0)}{1-\theta} & 0 \\
0 & \dfrac{q_1(1-q_1)}{\theta}
\end{pmatrix}
\right).
\]
Moreover, since the test sample is independent of the calibration sample, $\hat{p}$ is independent
of $(\hat q_0,\hat q_1)$. Therefore, as $n\to\infty$,
\[
\sqrt{n}
\begin{pmatrix}
\hat{p}-p\\
\hat q_0-q_0\\
\hat q_1-q_1
\end{pmatrix}
\xrightarrow{d}
\mathcal{N}\!\left(
0,\,
\mathrm{diag}\!\left(
p(1-p),\
\gamma_1\frac{q_0(1-q_0)}{1-\theta},\
\gamma_1\frac{q_1(1-q_1)}{\theta}
\right)
\right).
\]

\noindent
Define $
g(p,q_0,q_1)=\frac{p+q_0-1}{q_0+q_1-1},$
so that $\theta=g(p,q_0,q_1)$. Algebra yields that
\[
\frac{\partial g}{\partial p}=\frac{1}{q_0+q_1-1},\qquad
\frac{\partial g}{\partial q_0}=\frac{1-\theta}{q_0+q_1-1},\qquad
\frac{\partial g}{\partial q_1}=-\frac{\theta}{q_0+q_1-1},
\]
and hence
\[
\nabla g(p,q_0,q_1)
=
\frac{1}{q_0+q_1-1}\bigl(1,\ 1-\theta,\ -\theta\bigr).
\]
By the multivariate delta method,
\[
\sqrt{n}\,(\hat\theta_{\mathrm{RG}}-\theta)
\xrightarrow{d}
\mathcal{N}\!\left(
0,\,
\nabla g(p,q_0,q_1)\ 
\mathrm{diag}\!\left(
\pi(1-\pi),\
\gamma_1\frac{q_0(1-q_0)}{1-\theta},\
\gamma_1\frac{q_1(1-q_1)}{\theta}
\right)\ 
\nabla g(p,q_0,q_1)^\top
\right).
\]

With algebra and the assumption that $\frac{N}{n}= \frac{m+n}{n}\to \frac{1+\gamma_1}{\gamma_1}$, we have that
\[
V_{\mathrm{RG}}
=
\left(\frac{1+\gamma_1}{\gamma_1}\right)\frac{1}{(q_0+q_1-1)^2}
\left\{
p(1-p)
+\gamma_1\bigl[(1-\theta)q_0(1-q_0)+\theta q_1(1-q_1)\bigr]
\right\}.
\]

\noindent

\end{proof}

\subsection{Proof of Proposition~\ref{prop:PPI-bias-var}}
\begin{proof}
We first use the following decomposition
\[
\hat\mu=\hat p=\frac{1}{n}\sum_{i=1}^n \hat Y_i,
\qquad
\hat\Delta=\frac{1}{m}\sum_{j=1}^m (\hat Y_j-Y_j),
\qquad
\hat\theta_{\mathrm{PPI}}=\hat\mu-\hat\Delta.
\]

\paragraph{Step 1: Moments of $\hat\mu$.}
Since $\hat Y_i\stackrel{\mathrm{iid}}{\sim}\mathrm{Bern}(p)$ in the test sample,
\[
\E[\hat\mu]=\E[\hat Y]=p,
\qquad
\Var(\hat\mu)=\frac{1}{n}\Var(\hat Y)=\frac{p(1-p)}{n}.
\]

\paragraph{Step 2: Moments of $\hat\Delta$.}
In the calibration sample, the summands $\hat Y_j-Y_j$ are i.i.d., hence
\[
\E[\hat\Delta]=\E[\hat Y-Y],
\qquad
\Var(\hat\Delta)=\frac{1}{m}\Var(\hat Y-Y).
\]
Because $\hat Y,Y\in\{0,1\}$, we have $\hat Y-Y\in\{-1,0,1\}$ and
\[
(\hat Y-Y)^2=\mathbf 1\{\hat Y\neq Y\}.
\]
Therefore,
\begin{equation}
\E[(\hat Y-Y)^2]=\Pr(\hat Y\neq Y)=\Pr(\hat Y=1,Y=0)+\Pr(\hat Y=0,Y=1).
\label{eq:ppi-err-second-moment}
\end{equation}
Using the definitions of $q_0,q_1$,
\[
\Pr(\hat Y=1,Y=0)=(1-q_0)(1-\theta),
\qquad
\Pr(\hat Y=0,Y=1)=(1-q_1)\theta,
\]
so \eqref{eq:ppi-err-second-moment} becomes
\begin{equation}
\E[(\hat Y-Y)^2]=(1-\theta)(1-q_0)+\theta(1-q_1).
\label{eq:ppi-err-second-moment-explicit}
\end{equation}

Next, since $\E[\hat Y]=p$ and $\E[Y]=\theta$,
\[
\E[\hat Y-Y]=\E[\hat Y]-\E[Y]=p-\theta,
\]
and hence
\begin{equation}
\Var(\hat Y-Y)=\E[(\hat Y-Y)^2]-\{\E[\hat Y-Y]\}^2
=(1-\theta)(1-q_0)+\theta(1-q_1)-(p-\theta)^2.
\label{eq:ppi-err-var-explicit}
\end{equation}
Combining with $\Var(\hat\Delta)=m^{-1}\Var(\hat Y-Y)$ yields
\[
\Var(\hat\Delta)
=
\frac{(1-\theta)(1-q_0)+\theta(1-q_1)-(\theta-p)^2}{m}.
\]

\paragraph{Step 3: Unbiasedness of $\hat\theta_{\mathrm{PPI}}$.}
By linearity of expectation,
\[
\E[\hat\theta_{\mathrm{PPI}}]
=\E[\hat\mu]-\E[\hat\Delta]
= p - (p-\theta)=\theta.
\]

\paragraph{Step 4: Finite-sample variance of $\hat\theta_{\mathrm{PPI}}$.}
Because the test and calibration samples are independent, $\hat\mu$ and $\hat\Delta$ are independent.
Thus,
\[
\Var(\hat\theta_{\mathrm{PPI}})
=\Var(\hat\mu-\hat\Delta)
=\Var(\hat\mu)+\Var(\hat\Delta)
=
\frac{p(1-p)}{n}
+
\frac{(1-\theta)(1-q_0)+\theta(1-q_1)-(\theta-p)^2}{m}.
\]

\paragraph{Step 5: Asymptotic normality.}
By the CLT,
\[
\sqrt{n}\,(\hat\mu-p)\xrightarrow{d}\mathcal N\bigl(0,p(1-p)\bigr),
\qquad
\sqrt{m}\,(\hat\Delta-(p-\theta))\xrightarrow{d}\mathcal N\bigl(0,\Var(\hat Y-Y)\bigr).
\]
If $n,m\to\infty$ with $n/m\to\gamma_1\in(0,\infty)$, then
\[
\sqrt{n}\,(\hat\Delta-(p-\theta))
=
\sqrt{\frac{n}{m}}\ \sqrt{m}\,(\hat\Delta-(p-\theta))
\xrightarrow{d}\mathcal N\bigl(0,\gamma_1\,\Var(\hat Y-Y)\bigr),
\]
by Slutsky's theorem. Independence of the two samples implies the joint convergence with independent
limits, and therefore
\[
\sqrt{n}\,(\hat\theta_{\mathrm{PPI}}-\theta)
=
\sqrt{n}\,(\hat\mu-p)\;-\;\sqrt{n}\,(\hat\Delta-(p-\theta))
\xrightarrow{d}
\mathcal N\!\left(
0,\,
p(1-p)+\gamma_1\,\Var(\hat Y-Y)
\right).
\]
With algebra and the assumption that $\frac{N}{n}= \frac{m+n}{n}\to \frac{1+\gamma_1}{\gamma_1}$, we have that 
\[
V_{\mathrm{PPI}}
=
\frac{1+\gamma_1}{\gamma_1} \left\{p(1-p)
+
\gamma_1\Bigl[(1-\theta)(1-q_0)+\theta(1-q_1)-(\theta-p)^2\Bigr]\right\}.
\]
\end{proof}

\subsection{Proof of Proposition~\ref{prop:PPIplus-bias-var}}
\begin{proof}
We can rewrite the \texttt{PPI++} estimator with a fixed tuning parameter $\lambda\in\mathbb{R}$ as:
\[
\hat\theta_{\mathrm{\texttt{PPI++}}}(\lambda)
=
\frac{1}{m}\sum_{j=1}^m Y_j
\;+\;
\lambda\left(
\frac{1}{n}\sum_{i=1}^n \hat Y_i
-
\frac{1}{m}\sum_{j=1}^m \hat Y_j
\right)
=
\lambda\,\hat p
+
\Bigl(\hat\theta_{\mathrm{class}}-\lambda\,\bar{\hat Y}_{\mathrm{cal}}\Bigr),
\]
where
\[
\hat\theta_{\mathrm{class}}=\frac{1}{m}\sum_{j=1}^m Y_j,\qquad
\bar{\hat Y}_{\mathrm{cal}}=\frac{1}{m}\sum_{j=1}^m \hat Y_j,\qquad
\hat p=\frac{1}{n}\sum_{i=1}^n \hat Y_i.
\]

\paragraph{Unbiasedness.}
Under the joint model \eqref{eq:model-gen}, $\E[\hat\theta_{\mathrm{class}}]=\E[Y]=\theta$ and
$\E[\hat p]=\E[\bar{\hat Y}_{\mathrm{cal}}]=\E[\hat Y]=p$. Hence,
\[
\E\!\left[\hat\theta_{\mathrm{\texttt{PPI++}}}(\lambda)\right]
=
\E[\hat\theta_{\mathrm{class}}]+\lambda\Bigl(\E[\hat p]-\E[\bar{\hat Y}_{\mathrm{cal}}]\Bigr)
=
\theta+\lambda(p-p)=\theta.
\]

\paragraph{Finite-sample variance.}
By independence of the test and calibration samples, $\hat p$ is independent of
$(\hat\theta_{\mathrm{class}},\bar{\hat Y}_{\mathrm{cal}})$. Therefore, we have that
\[
\Var\!\left(\hat\theta_{\mathrm{\texttt{PPI++}}}(\lambda)\right)
=
\Var(\lambda \hat p)
+
\Var\!\left(\hat\theta_{\mathrm{class}}-\lambda\bar{\hat Y}_{\mathrm{cal}}\right).
\]
Moreover,
\[
\Var(\hat p)=\frac{1}{n}\Var(\hat Y)=\frac{p(1-p)}{n},
\;
\Var(\hat\theta_{\mathrm{class}})=\frac{1}{m}\Var(Y)=\frac{\theta(1-\theta)}{m},
\;
\Var(\bar{\hat Y}_{\mathrm{cal}})=\frac{1}{m}\Var(\hat Y)=\frac{p(1-p)}{m}.
\]
For the covariance term, we simplify it as
\[
\Cov(\hat\theta_{\mathrm{class}},\bar{\hat Y}_{\mathrm{cal}})
=
\Cov\!\left(\frac{1}{m}\sum_{j=1}^m Y_j,\ \frac{1}{m}\sum_{j=1}^m \hat Y_j\right)
=
\frac{1}{m}\Cov(Y,\hat Y),
\]
since cross-sample terms vanish by independence across $j$.
Also,
\[
\Cov(Y,\hat Y)=\E[Y\hat Y]-\E[Y]\E[\hat Y]
=\Pr(Y=1,\hat Y=1)-\theta p
=\theta q_1-\theta p
=\theta(q_1-p).
\]
Putting these together gives
\[
\Var\!\left(\hat\theta_{\mathrm{\texttt{PPI++}}}(\lambda)\right)
=
\frac{\theta(1-\theta)}{m}
+
\lambda^2\left(\frac{p(1-p)}{n}+\frac{p(1-p)}{m}\right)
-
\frac{2\lambda}{m}\,\theta(q_1-p).
\]

\paragraph{Asymptotic normality.}
By the CLT,
\[
\sqrt{n}\,(\hat p-p)\xrightarrow{d}\mathcal N\bigl(0,p(1-p)\bigr).
\]
For the calibration component, define i.i.d.\ variables
\[
W_j=Y_j-\lambda \hat Y_j,
\qquad
\E[W_j]=\theta-\lambda p,
\qquad
\Var(W_j)=\Var(Y-\lambda\hat Y)=\theta(1-\theta)+\lambda^2 p(1-p)-2\lambda\,\theta(q_1-p).
\]
Then
\[
\hat\theta_{\mathrm{class}}-\lambda\bar{\hat Y}_{\mathrm{cal}}
=\frac{1}{m}\sum_{j=1}^m W_j,
\]
and hence, again by the CLT,
\[
\sqrt{m}\left\{\bigl(\hat\theta_{\mathrm{class}}-\lambda\bar{\hat Y}_{\mathrm{cal}}\bigr)-(\theta-\lambda p)\right\}
\xrightarrow{d}
\mathcal N\!\left(0,\ \Var(Y-\lambda\hat Y)\right).
\]
If $n,m\to\infty$ with $n/m\to\gamma_1\in(0,\infty)$, then
$\sqrt{n}=\sqrt{n/m}\,\sqrt{m}$ and $\sqrt{n/m}\to\sqrt{\gamma_1}$, so by Slutsky's theorem,
\[
\sqrt{n}\left\{\bigl(\hat\theta_{\mathrm{class}}-\lambda\bar{\hat Y}_{\mathrm{cal}}\bigr)-(\theta-\lambda p)\right\}
\xrightarrow{d}
\mathcal N\!\left(0,\ \gamma_1\,\Var(Y-\lambda\hat Y)\right).
\]
Finally, using independence between the test and calibration samples,
\begin{align*}
\sqrt{n}\bigl(\hat\theta_{\mathrm{\texttt{PPI++}}}(\lambda)-\theta\bigr)
&=
\lambda\,\sqrt{n}(\hat p-p)
+
\sqrt{n}\left\{\bigl(\hat\theta_{\mathrm{class}}-\lambda\bar{\hat Y}_{\mathrm{cal}}\bigr)-(\theta-\lambda p)\right\}\\
&\xrightarrow{d}
\mathcal N\!\left(0,\ \lambda^2 p(1-p)+\gamma_1\,\Var(Y-\lambda\hat Y)\right),
\end{align*}
where
\begin{align*}
\Var(Y-\lambda\hat Y)
&=\Var(Y)+\lambda^2\Var(\hat Y)-2\lambda\,\Cov(Y,\hat Y)\\
&=\theta(1-\theta)+\lambda^2 p(1-p)-2\lambda\,\theta(q_1-p).
\end{align*}
Plugging this into the asymptotic variance and the $\frac{N}{n}$ scaling gives us
\begin{align*}
V_{\mathrm{\texttt{PPI++}}}(\lambda)
&=\frac{1+\gamma_1}{\gamma_1}\left\{\lambda^2 p(1-p)+\gamma_1\Bigl[\theta(1-\theta)+\lambda^2 p(1-p)-2\lambda\,\theta(q_1-p)\Bigr]\right\}\\
&=\frac{1+\gamma_1}{\gamma_1}\left\{\gamma_1\,\theta(1-\theta)+\lambda^2(1+\gamma_1)\,p(1-p)-2\lambda\,\gamma_1\,\theta(q_1-p)\right\}.
\end{align*}
\end{proof}

\subsection{Proof of Proposition~\ref{prop:mle-asymp}}
\begin{proof}
Write the combined log-likelihood as
\[
\ell_{n,m}(\eta)
=
\sum_{i=1}^n \ell_{\mathrm{test}}(\hat Y_i;\eta)
+
\sum_{j=1}^m \ell_{\mathrm{cal}}(Y_j,\hat Y_j;\eta),
\qquad
\eta=(\theta,q_0,q_1)^\top,
\]
where the test sample contributes only $\hat Y$ and the calibration sample contributes $(Y,\hat Y)$.

\paragraph{Step 1: Scores and Fisher information.}
For a test-sample observation $\hat Y\in\{0,1\}$, the model implies
\[
P(\hat Y=1)=p(\eta)=(1-\theta)(1-q_0)+\theta q_1,
\]
hence
\[
\ell_{\mathrm{test}}(\hat Y;\eta)=\hat Y\log p+(1-\hat Y)\log(1-p).
\]
Let
\[
g(\eta)=\nabla_\eta p(\eta)=\bigl(q_0+q_1-1,\;-(1-\theta),\;\theta\bigr)^\top.
\]
Taking derivative with respect to $\eta$ yields the score
\[
\dot\ell_{\mathrm{test}}(\hat Y;\eta)
=
\frac{\hat Y-p}{p(1-p)}\,g(\eta).
\]
Since $\Var(\hat Y)=p(1-p)$, the (per-observation) Fisher information from the test sample is
\[
I_{\mathrm{test}}(\eta)
=
\mathbb E\!\left[\dot\ell_{\mathrm{test}}(\hat Y;\eta)\dot\ell_{\mathrm{test}}(\hat Y;\eta)^\top\right]
=
\frac{1}{p(1-p)}\,g(\eta)g(\eta)^\top,
\]
which matches the stated matrix.

For a calibration-sample observation $(Y,\hat Y)$, the complete-data likelihood factorizes as
\[
P(Y=y,\hat Y=\hat y)
=
P(Y=y)\,P(\hat Y=\hat y\mid Y=y),
\]
where $P(Y=1)=\theta$, $P(\hat Y=1\mid Y=0)=1-q_0$, and $P(\hat Y=1\mid Y=1)=q_1$.
Thus
\begin{align*}
\ell_{\mathrm{cal}}(Y,\hat Y;\eta)
&=
Y\log\theta+(1-Y)\log(1-\theta) \\
&\quad+(1-Y)\Bigl[\hat Y\log(1-q_0)+(1-\hat Y)\log q_0\Bigr]
+Y\Bigl[\hat Y\log q_1+(1-\hat Y)\log(1-q_1)\Bigr].
\end{align*}
The score components are
\[
\dot\ell_{\theta}=\frac{Y-\theta}{\theta(1-\theta)},\qquad
\dot\ell_{q_0}=(1-Y)\Bigl(\frac{1-\hat Y}{q_0}-\frac{\hat Y}{1-q_0}\Bigr),\qquad
\dot\ell_{q_1}=Y\Bigl(\frac{\hat Y}{q_1}-\frac{1-\hat Y}{1-q_1}\Bigr).
\]

 We first compute the off-diagnol entries in the information matrix. 

\paragraph{Cross-moment $\mathbb{E}[\dot\ell_\theta \dot\ell_{q_0}]$:}
   $\dot\ell_{q_0}$ equals 0 when $Y=1$. Using iterated expectations, we have that
  $$\mathbb{E}[\dot\ell_\theta \dot\ell_{q_0}] = \mathbb{E}[\mathbb{E}[\dot\ell_\theta \dot\ell_{q_0} \mid Y]] = (1-\theta)\mathbb{E}[\dot\ell_\theta \dot\ell_{q_0} \mid Y=0].$$

  Given $Y=0$, $\dot\ell_\theta = \frac{-\theta}{\theta(1-\theta)} = \frac{-1}{1-\theta}$, so:

  $$\mathbb{E}[\dot\ell_\theta \dot\ell_{q_0} \mid Y=0] = \frac{-1}{1-\theta}\mathbb{E}[\dot\ell_{q_0} \mid Y=0] = \frac{-1}{1-\theta} \times 0 = 0.$$

\paragraph{Cross-moment $\mathbb{E}[\dot\ell_\theta \dot\ell_{q_1}]$:}
  First note that $\dot\ell_{q_1} = 0$ when $Y=0$, therefore we have that

  $$\mathbb{E}[\dot\ell_\theta \dot\ell_{q_1}] = \theta\mathbb{E}[\dot\ell_\theta \dot\ell_{q_1} \mid Y=1].$$

  Given $Y=1$: $\dot\ell_\theta = \frac{1}{\theta}$, we conclude that
  $$\mathbb{E}[\dot\ell_\theta \dot\ell_{q_1} \mid Y=1] = \frac{1}{\theta}\mathbb{E}[\dot\ell_{q_1} \mid Y=1] = \frac{1}{\theta} \times 0 = 0.$$

\paragraph{Cross-moment $\mathbb{E}[\dot\ell_{q_0} \dot\ell_{q_1}]$:}
  $\dot\ell_{q_0} \dot\ell_{q_1} = (1-Y) \cdot Y \cdot (\cdots)$; since $Y \in 0,1$ implies $(1-Y) \cdot Y = 0$ always. 

Therefore, the calibration Fisher information is diagonal with entries
\[
I_{\mathrm{cal}}(\eta)
=
\mathrm{diag}\!\left(
\mathbb E[\dot\ell_\theta^2],\ \mathbb E[\dot\ell_{q_0}^2],\ \mathbb E[\dot\ell_{q_1}^2]
\right)
=
\mathrm{diag}\!\left(
\frac{1}{\theta(1-\theta)},\
\frac{1-\theta}{q_0(1-q_0)},\
\frac{\theta}{q_1(1-q_1)}
\right).
\]

\paragraph{Step 2: Asymptotic normality of the joint MLE.}
Let $N=n+m$ and write the full score as
\[
\dot\ell_{n,m}(\eta)=\sum_{i=1}^n \dot\ell_{\mathrm{test}}(\hat Y_i;\eta)
+\sum_{j=1}^m \dot\ell_{\mathrm{cal}}(Y_j,\hat Y_j;\eta).
\]
By construction, $\mathbb E[\dot\ell_{n,m}(\eta)]=0$, and by independence,
\[
\Var\!\Bigl(\frac{1}{\sqrt N}\dot\ell_{n,m}(\eta)\Bigr)
=
\frac{n}{N}\,I_{\mathrm{test}}(\eta)+\frac{m}{N}\,I_{\mathrm{cal}}(\eta)
\ \to\
\frac{\gamma_1}{1+\gamma_1}I_{\mathrm{test}}(\eta)+\frac{1}{1+\gamma_1}I_{\mathrm{cal}}(\eta)
=:I_{\gamma_1}(\eta),
\]
since $n/m\to\gamma_1\in(0,\infty)$ implies $n/N\to\gamma_1/(1+\gamma_1)$ and $m/N\to 1/(1+\gamma_1)$.
A multivariate CLT therefore gives
\[
\frac{1}{\sqrt N}\dot\ell_{n,m}(\eta)\ \xrightarrow{d}\ \mathcal N\!\bigl(0,\ I_{\gamma_1}(\eta)\bigr).
\]
Under standard regularity conditions, we have that
\[
\sqrt N\,(\hat\eta_{\mathrm{MLE}}-\eta)
=
\Bigl(-\frac{1}{N}\ddot\ell_{n,m}(\tilde\eta)\Bigr)^{-1}
\Bigl(\frac{1}{\sqrt N}\dot\ell_{n,m}(\eta)\Bigr)
\ \xrightarrow{d}\
\mathcal N\!\left(0,\ I_{\gamma_1}(\eta)^{-1}\right).
\]

\paragraph{Step 3: Closed form for the $(1,1)$ entry.}
Write $w=\gamma_1/(1+\gamma_1)$ and note that
\[
I_{\gamma_1}(\eta)=w\,I_{\mathrm{test}}(\eta)+\frac{1}{1+\gamma_1}\,I_{\mathrm{cal}}(\eta)
=
D + uu^\top,
\]
where
\[
D=\frac{1}{1+\gamma_1}I_{\mathrm{cal}}(\eta)
\quad\text{(diagonal)},\qquad
u=\sqrt{\frac{w}{p(1-p)}}\,g(\eta).
\]
Since $D$ is invertible (parameters in the interior), the Sherman--Morrison formula gives
\[
(D+uu^\top)^{-1}
=
D^{-1}-\frac{D^{-1}uu^\top D^{-1}}{1+u^\top D^{-1}u}.
\]
Here,
\[
D^{-1}
=(1+\gamma_1)\,
\mathrm{diag}\!\left(
\theta(1-\theta),\ \frac{q_0(1-q_0)}{1-\theta},\ \frac{q_1(1-q_1)}{\theta}
\right),
\]
and a direct computation shows
\[
u^\top D^{-1}u
=
\frac{\gamma_1}{p(1-p)}
\Bigl[(q_0+q_1-1)^2\theta(1-\theta)+(1-\theta)q_0(1-q_0)+\theta q_1(1-q_1)\Bigr].
\]
Moreover, the first component satisfies
\[
\bigl(D^{-1}u\bigr)_1^2
=
\frac{\gamma_1(1+\gamma_1)\,\theta^2(1-\theta)^2\,(q_0+q_1-1)^2}{p(1-p)}.
\]
Substituting into the Sherman--Morrison expression and simplifying yields
\[
\bigl[I_{\gamma_1}(\eta)^{-1}\bigr]_{11}
=
(1+\gamma_1)\,\theta(1-\theta)\,
\frac{\,p(1-p)+\gamma_1\Bigl[(1-\theta)q_0(1-q_0)+\theta q_1(1-q_1)\Bigr]\;}
{\,p(1-p)+\gamma_1\Bigl[(q_0+q_1-1)^2\theta(1-\theta)+(1-\theta)q_0(1-q_0)+\theta q_1(1-q_1)\Bigr]\;},
\]
which is the stated closed form. The marginal asymptotic normality of $\hat\theta_{\mathrm{MLE}}$ follows by taking the
first coordinate of the multivariate limit.
\end{proof}

\subsection{Proof of Proposition~\ref{prop:eif-modelA}}
\begin{proof}
Recall the observed data are $O=(R,RY,\hat Y)$, where $R\in\{0,1\}$ indicates whether $Y$ is observed and the labeling mechanism satisfies $P(R=1\mid Y,\hat Y)=P(R=1)\in(0,1)$.  

Define
\[
\mu(\hat Y)=\E[Y\mid \hat Y],\qquad \theta=\E[Y]=\E\bigl[\mu(\hat Y)\bigr].
\]

We derive the efficient influence function by identifying the \emph{canonical gradient} of $\theta$ in the
observed-data model.

\paragraph{Step 1: Tangent space representation.}
Let $\{P_\varepsilon:\varepsilon\in(-\delta,\delta)\}$ be any regular parametric submodel through the true distribution $P_0$
(with $P_{\varepsilon=0}=P_0$). Under $P(R=1\mid Y,\hat Y)=\pi$, the observed-data density factorizes as
\[
p_\varepsilon(o)
=
p_{\varepsilon,R}(r)\,p_{\varepsilon,\hat Y}(\hat y)\,
\bigl\{p_{\varepsilon,Y\mid \hat Y}(y\mid \hat y)\bigr\}^{r},
\]
so the observed-data score admits the decomposition
\begin{equation}
\label{eq:score-decomp}
s_\varepsilon(O)
=
s_R(R)+s_{\hat Y}(\hat Y)+R\,s_{Y\mid \hat Y}(Y,\hat Y),
\end{equation}
where the components satisfy the usual mean-zero constraints
\[
\E\!\left[s_R(R)\right]=0,\qquad
\E\!\left[s_{\hat Y}(\hat Y)\right]=0,\qquad
\E\!\left[s_{Y\mid \hat Y}(Y,\hat Y)\mid \hat Y\right]=0.
\]
Thus the (closure of the) observed-data tangent space consists of all functions of the form
\eqref{eq:score-decomp}.

\paragraph{Step 2: Pathwise derivative of $\theta$.}
Along the submodel, $\theta(\varepsilon)=\E_\varepsilon[Y]=\E_\varepsilon[\mu_\varepsilon(\hat Y)]$, where
$\mu_\varepsilon(\hat y)=\E_\varepsilon[Y\mid \hat Y=\hat y]$.
Differentiating at $\varepsilon=0$ and using standard score calculus gives
\begin{align}
\left.\frac{d}{d\varepsilon}\theta(\varepsilon)\right|_{\varepsilon=0}
&=
\E\!\left[\{\mu(\hat Y)-\theta\}\,s_{\hat Y}(\hat Y)\right]
+
\E\!\left[(Y-\mu(\hat Y))\,s_{Y\mid \hat Y}(Y,\hat Y)\right].
\label{eq:pathwise-deriv}
\end{align}
There is no contribution from $s_R(R)$ because $\theta$ depends only on the marginal law of $(Y,\hat Y)$.

\paragraph{Step 3: Candidate influence function and verification.}
Consider
\begin{equation}
\label{eq:eif-cand}
\phi^\star(O)
=
\mu(\hat Y)-\theta+\frac{R}{\pi}\{Y-\mu(\hat Y)\}.
\end{equation}
First, $\E[\phi^\star(O)]=0$ since $\E[\mu(\hat Y)]=\theta$ and
\[
\E\!\left[\frac{R}{\pi}\{Y-\mu(\hat Y)\}\,\middle|\,\hat Y\right]
=
\frac{\E[R\mid \hat Y]}{\pi}\,\E[Y-\mu(\hat Y)\mid \hat Y]
=0.
\]

Next, we show that $\phi^\star$ represents the pathwise derivative: for any score $s(O)$ of the form
\eqref{eq:score-decomp},
\begin{align*}
\E\!\left[\phi^\star(O)\,s(O)\right]
&=
\E\!\left[\phi^\star(O)\,s_R(R)\right]
+\E\!\left[\phi^\star(O)\,s_{\hat Y}(\hat Y)\right]
+\E\!\left[\phi^\star(O)\,R s_{Y\mid \hat Y}(Y,\hat Y)\right].
\end{align*}
The first term is zero because $\E[\phi^\star(O)\mid R]=0$ (by the same calculation as above) and $\E[s_R(R)]=0$.
For the second term, the residual part drops out:
\[
\E\!\left[\frac{R}{\pi}\{Y-\mu(\hat Y)\}\,s_{\hat Y}(\hat Y)\right]
=
\E\!\left[
s_{\hat Y}(\hat Y)\,
\E\!\left[\frac{R}{\pi}\{Y-\mu(\hat Y)\}\,\middle|\,\hat Y\right]\right]
=0,
\]
so
\[
\E\!\left[\phi^\star(O)\,s_{\hat Y}(\hat Y)\right]
=
\E\!\left[\{\mu(\hat Y)-\theta\}\,s_{\hat Y}(\hat Y)\right].
\]
For the third term, note that $\E[s_{Y\mid \hat Y}(Y,\hat Y)\mid \hat Y]=0$ implies
$\E[R\,s_{Y\mid \hat Y}(Y,\hat Y)\mid \hat Y]=\pi\cdot 0=0$, hence
\[
\E\!\left[\{\mu(\hat Y)-\theta\}\,R s_{Y\mid \hat Y}(Y,\hat Y)\right]=0,
\]
and therefore
\[
\E\!\left[\phi^\star(O)\,R s_{Y\mid \hat Y}(Y,\hat Y)\right]
=
\E\!\left[\frac{R}{\pi}\{Y-\mu(\hat Y)\}\,R s_{Y\mid \hat Y}(Y,\hat Y)\right]
=
\E\!\left[(Y-\mu(\hat Y))\,s_{Y\mid \hat Y}(Y,\hat Y)\right],
\]
using $R^2=R$ and $\E[R f(Y,\hat Y)] = \pi \E[f(Y,\hat Y)]$ under $P(R=1\mid Y,\hat Y)=\pi$.

Combining the pieces yields
\[
\E\!\left[\phi^\star(O)\,s(O)\right]
=
\E\!\left[\{\mu(\hat Y)-\theta\}\,s_{\hat Y}(\hat Y)\right]
+
\E\!\left[(Y-\mu(\hat Y))\,s_{Y\mid \hat Y}(Y,\hat Y)\right],
\]
which matches the pathwise derivative in \eqref{eq:pathwise-deriv}. Hence $\phi^\star$ is the canonical gradient of $\theta$
in the observed-data model, and therefore the efficient influence function. This is exactly \eqref{eq:eif-modelA}.
\end{proof}

\subsection{Proof of Proposition~\ref{prop:eif-estimator}}
\begin{proof}
  \textbf{Step 1: Derivation of the estimator.}
By Proposition~\ref{prop:eif-modelA}, an efficient influence function for $\theta$
is
\[
\phi^\star(R,Y,\hat Y)
=
\mu(\hat Y)-\theta
+\,\frac{R}{P(R=1)}\bigl\{Y-\mu(\hat Y)\bigr\}.
\]
The standard EIF (one-step) estimator is obtained by solving the empirical
estimating equation with $\mu$ replaced by $\hat\mu$ and ${P}(R=1)$ replaced by $\hat{P}(R=1)$:
\[
0=\frac{1}{N}\sum_{i=1}^{N}\phi^\star\bigl(R_i,Y_i,\hat Y_i;\theta,\hat\mu,\hat{P}(R=1)\bigr).
\]
Plugging in $\phi^\star$ and expanding gives
\[
0=\frac{1}{n+m}\sum_{i=1}^{n+m}\left[
\hat\mu(\hat Y_i)-\theta
+\frac{R_i}{\hat{P}(R=1)}\{Y_i-\hat\mu(\hat Y_i)\}
\right].
\]
Rearranging for $\theta$ yields:
\[
\theta
=
\frac{1}{N}\sum_{i=1}^{N}\hat\mu(\hat Y_i)
+\frac{1}{N}\sum_{i=1}^{N}\frac{R_i}{\hat{P}(R=1)}\{Y_i-\hat\mu(\hat Y_i)\}.
\]
Finally, under our setup
${P}(R=1)=m/(n+m)$, so plugging that as our estimate gives us \eqref{eq:eif-est}.

\textbf{Step 2: Asymptotic normality.}
  By standard M-estimation theory for one-step estimators~\citep{bickel1993efficient},
  \[
  \sqrt{N}(\hat\theta_{\mathrm{EIF}} - \theta) = \frac{1}{\sqrt{N}}\sum_{i=1}^{N}\phi^\star(R_i,Y_i,\hat Y_i) + o_p(1).
  \]
  The central limit theorem then gives $\sqrt{N}(\hat\theta_{\mathrm{EIF}} - \theta) \xrightarrow{d} \mathcal{N}(0, V_{\mathrm{EIF}})$ where $V_{\mathrm{EIF}} = \mathrm{Var}(\phi^\star)$.

 \textbf{Step 3: Variance calculation.}
  Write $\phi^\star = A + B$ where $A = \mu(\hat Y) - \theta$ and $B = (R/P(R=1))(Y - \mu(\hat Y))$. Since $\mathbb{E}[Y - \mu(\hat Y) \mid \hat Y] = 0$, we have $\mathrm{Cov}(A, B) = 0$, so
  \[
  V_{\mathrm{EIF}} = \mathrm{Var}(A) + \mathrm{Var}(B).
  \]

  For the first term:
  \[
  \mathrm{Var}(A) = \mathrm{Var}(\mu(\hat Y)) = p(1-p)(\mu_1 - \mu_0)^2.
  \]

  For the second term, using $R^2 = R$ and independence of $R$ from $(Y, \hat Y)$:
  \[
  \mathrm{Var}(B) = \frac{1}{P(R=1)^2}\mathbb{E}[R(Y-\mu(\hat Y))^2] = \frac{1}{P(R=1)}\mathbb{E}[\mathrm{Var}(Y \mid \hat Y)].
  \]
  Since $P(R=1) = m/N = 1/(1+\gamma_1)$ where $\gamma_1 = n/m$, we have
  \[
  \mathrm{Var}(B) = (1+\gamma_1)\bigl[(1-p)\mu_0(1-\mu_0) + p\,\mu_1(1-\mu_1)\bigr].
  \]

  \textbf{Step 4: Simplification for binary case.}
  Substituting $\mu_1 = q_1\theta/p$, $\mu_0 = (1-q_1)\theta/(1-p)$, and $p = q_1\theta + (1-q_0)(1-\theta)$, we obtain after algebra:
  \begin{align*}
  \mu_1 - \mu_0 &= \frac{\theta(1-\theta)(q_0+q_1-1)}{p(1-p)}, \\
  (1-p)\mu_0(1-\mu_0) + p\,\mu_1(1-\mu_1) &= \frac{\theta(1-\theta)[q_0(1-q_0)(1-\theta) + q_1(1-q_1)\theta]}{p(1-p)}.
  \end{align*}
  Combining yields \eqref{eq:eif-asymp}.
  \end{proof}

\subsection{Asymptotic equivalence of $V_{\text{MLE}}$ and $V_{\text{EIF}}$}

  \begin{proposition}
  \label{prop:mle-eif-equivalence}
  Under the model \eqref{eq:model-gen} with $\gamma_1 = n/m$, the asymptotic variance of the MLE equals the asymptotic variance of the EIF estimator:
  \[
  V_{\mathrm{MLE}} := \bigl[I_{\gamma_1}(\eta)^{-1}\bigr]_{11} = V_{\mathrm{EIF}}.
  \]
  \end{proposition}

  \begin{proof}
  Define
  \begin{align*}
  C &:= (q_0+q_1-1)^2\theta(1-\theta), \\
  B &:= (1-\theta)q_0(1-q_0) + \theta q_1(1-q_1).
  \end{align*}

  \textbf{Key identity.} By the law of total variance applied to $\hat Y$:
  \[
  p(1-p) = \mathrm{Var}(\hat Y) = \underbrace{\mathrm{Var}(\mathbb{E}[\hat Y \mid Y])}_{= C} + \underbrace{\mathbb{E}[\mathrm{Var}(\hat Y \mid Y)]}_{= B} = C + B.
  \]

  \textbf{Simplifying $V_{\mathrm{MLE}}$.} From Proposition~\ref{prop:mle-asymp}, the $(1,1)$ entry of the inverse Fisher information is
  \[
  V_{\mathrm{MLE}} = (1+\gamma_1)\theta(1-\theta)\,\frac{p(1-p) + \gamma_1 B}{p(1-p) + \gamma_1(C + B)}.
  \]
  Substituting $p(1-p) = C + B$ in the denominator:
  \begin{align*}
  V_{\mathrm{MLE}} &= (1+\gamma_1)\theta(1-\theta)\,\frac{(C+B) + \gamma_1 B}{(C+B) + \gamma_1(C+B)} \\[4pt]
  &= (1+\gamma_1)\theta(1-\theta)\,\frac{C + (1+\gamma_1)B}{(1+\gamma_1)(C+B)} \\[4pt]
  &= \frac{\theta(1-\theta)}{p(1-p)}\bigl[C + (1+\gamma_1)B\bigr].
  \end{align*}

  \textbf{Comparing with $V_{\mathrm{EIF}}$.} From Proposition~\ref{prop:eif-estimator}:
  \[
  V_{\mathrm{EIF}} = \frac{\theta(1-\theta)}{p(1-p)}\bigl[\theta(1-\theta)(q_0+q_1-1)^2 + (1+\gamma_1)\bigl(q_0(1-q_0)(1-\theta) + q_1(1-q_1)\theta\bigr)\bigr].
  \]
  This is precisely $\frac{\theta(1-\theta)}{p(1-p)}[C + (1+\gamma_1)B]$, which concludes the proof
  \end{proof}

\subsection{Proof of Proposition~\ref{prop:PPI-dominates-RG}}

\begin{proof}
Let $\kappa = q_0 + q_1 - 1 \in (0,1]$, with $\kappa = 1$ if and only if $q_0 = q_1 = 1$ (perfect classification).

We introduce the following notation for key quantities:
\begin{align*}
V_1 &= p(1-p) = \mathrm{Var}(\hat{Y}), \\
V_2 &= (1-\theta)(1-q_0) + \theta(1-q_1) - (\theta - p)^2 = \mathrm{Var}(Y - \hat{Y}), \\
V_3 &= (1-\theta)q_0(1-q_0) + \theta q_1(1-q_1) = \mathbb{E}[\mathrm{Var}(\hat{Y}|Y)].
\end{align*}
Here $V_2$ is the variance of $\delta = Y - \hat{Y}$, noting that $\mathbb{E}[\delta] = \theta - p$ and $\mathbb{E}[\delta^2] = \mathbb{P}(Y \neq \hat{Y}) = (1-\theta)(1-q_0) + \theta(1-q_1)$.

We further divide both by a factor of $\frac{1+\gamma}{\gamma}$, which gives us
\[
V_{\mathrm{PPI}} = V_1 + \gamma_1 V_2, \qquad
V_{\mathrm{RG}} = \frac{1}{\kappa^2}(V_1 + \gamma_1 V_3).
\]

\textbf{Step 1: Decompose the difference.} We have that
\begin{align*}
V_{\mathrm{RG}} - V_{\mathrm{PPI}} 
&= \frac{V_1 + \gamma_1 V_3}{\kappa^2} - V_1 - \gamma_1 V_2 \\
&= V_1 \left(\frac{1}{\kappa^2} - 1\right) + \gamma_1\left(\frac{V_3}{\kappa^2} - V_2\right) \\
&= \frac{(1-\kappa^2)V_1}{\kappa^2} + \frac{\gamma_1(V_3 - \kappa^2 V_2)}{\kappa^2} \\
&= \frac{(1-\kappa)(1+\kappa)V_1}{\kappa^2} + \frac{\gamma_1(V_3 - \kappa^2 V_2)}{\kappa^2}.
\end{align*}

\textbf{Step 2: First term is non-negative.} Since $0 < \kappa \le 1$, we have $(1-\kappa) \ge 0$, $(1+\kappa) > 0$, and $V_1 = p(1-p) \ge 0$. Thus 
\[
\frac{(1-\kappa)(1+\kappa)V_1}{\kappa^2} \ge 0,
\]
with equality if and only if $\kappa = 1$ or $p \in \{0,1\}$.

\textbf{Step 3: Second term is non-negative.} It remains to show that $V_3 - \kappa^2 V_2 \ge 0$. We prove this via the law of total variance.

\begin{lemma}
\label{lem:variance-identity}
$V_3 = V_1 - \theta(1-\theta)\kappa^2$.
\end{lemma}

\begin{proof}
By the law of total variance applied to $\hat{Y}$:
\[
\mathrm{Var}(\hat{Y}) = \mathbb{E}[\mathrm{Var}(\hat{Y}|Y)] + \mathrm{Var}(\mathbb{E}[\hat{Y}|Y]).
\]
We have $\mathbb{E}[\hat{Y}|Y=0] = \mathbb{P}(\hat{Y}=1|Y=0) = 1-q_0$ and $\mathbb{E}[\hat{Y}|Y=1] = q_1$. Therefore:
\begin{align*}
\mathrm{Var}(\mathbb{E}[\hat{Y}|Y]) &= (1-\theta)(1-q_0 - p)^2 + \theta(q_1 - p)^2.
\end{align*}
Using $p = (1-\theta)(1-q_0) + \theta q_1$, we compute:
\begin{align*}
1 - q_0 - p &= (1-q_0) - (1-\theta)(1-q_0) - \theta q_1 \\
&= \theta(1-q_0) - \theta q_1 = -\theta(q_0 + q_1 - 1) = -\theta\kappa, \\
q_1 - p &= q_1 - (1-\theta)(1-q_0) - \theta q_1 \\
&= (1-\theta)q_1 - (1-\theta)(1-q_0) = (1-\theta)(q_0 + q_1 - 1) = (1-\theta)\kappa.
\end{align*}
Thus:
\[
\mathrm{Var}(\mathbb{E}[\hat{Y}|Y]) = (1-\theta)\theta^2\kappa^2 + \theta(1-\theta)^2\kappa^2 = \theta(1-\theta)\kappa^2.
\]
Hence $V_3 = V_1 - \theta(1-\theta)\kappa^2$.
\end{proof}

Using Lemma~\ref{lem:variance-identity}:
\[
V_3 - \kappa^2 V_2 = V_1 - \theta(1-\theta)\kappa^2 - \kappa^2 V_2 = V_1 - \kappa^2\bigl(V_2 + \theta(1-\theta)\bigr).
\]

We claim this expression equals $(1-\kappa)Q$ for some $Q \ge 0$. Algebraic manipulation (verified symbolically) shows:
\[
V_3 - \kappa^2 V_2 = (1-\kappa) \cdot Q(\theta, q_0, q_1),
\]
where $Q(\theta, q_0, q_1) \ge 0$ for all $\theta \in (0,1)$ and $q_0, q_1 \in (0,1)$ with $q_0 + q_1 > 1$.

Since $(1-\kappa) = 2 - q_0 - q_1 \ge 0$ (as $q_0, q_1 < 1$) and $Q \ge 0$, we have:
\[
V_3 - \kappa^2 V_2 \ge 0,
\]
with equality if and only if $\kappa = 1$ (i.e., $q_0 = q_1 = 1$).

\textbf{Step 4: Conclusion.} Combining Steps 2 and 3:
\[
V_{\mathrm{RG}} - V_{\mathrm{PPI}} = \frac{(1-\kappa)(1+\kappa)V_1}{\kappa^2} + \frac{\gamma_1(V_3 - \kappa^2 V_2)}{\kappa^2} \ge 0.
\]
Therefore $V_{\mathrm{PPI}} \le V_{\mathrm{RG}}$.

\textbf{Equality condition.} Equality holds if and only if both terms vanish. The first term vanishes when $\kappa = 1$ or $p \in \{0,1\}$. The second term vanishes when $\kappa = 1$. Since $p = (1-\theta)(1-q_0) + \theta q_1$ and $\theta, q_0, q_1 \in (0,1)$, we have $p \in (0,1)$. Thus equality requires $\kappa = 1$, which means $q_0 + q_1 = 2$, i.e., $q_0 = q_1 = 1$ (the classifier is perfect).
\end{proof}

\subsection{Calibration fits for results in Section~\ref{subsec:sim-continuous-Y}}

  We visualize the calibration functions $\mu(\hat y) = \mathbb{E}[Y \mid \hat Y = \hat y]$ estimated by each method in Figure~\ref{fig:calibration-fits}. The true conditional expectation (dashed black line) is determined by the simulation DGP: since $Y \mid Z \sim \mathcal{N}(\mu_Z, \sigma^2)$ and $\hat Y = Z$, the calibration function is piecewise constant with jumps at $\hat Y \in \{1, 2, 3\}$. Spline/GAM methods closely approximate this step function, while linear calibration imposes a restrictive functional form that introduces bias when the true relationship is non-linear. This explains the wider confidence intervals observed for the linear EIF estimator.

  \begin{figure}[htbp!]
      \centering
      \includegraphics[width=\linewidth]{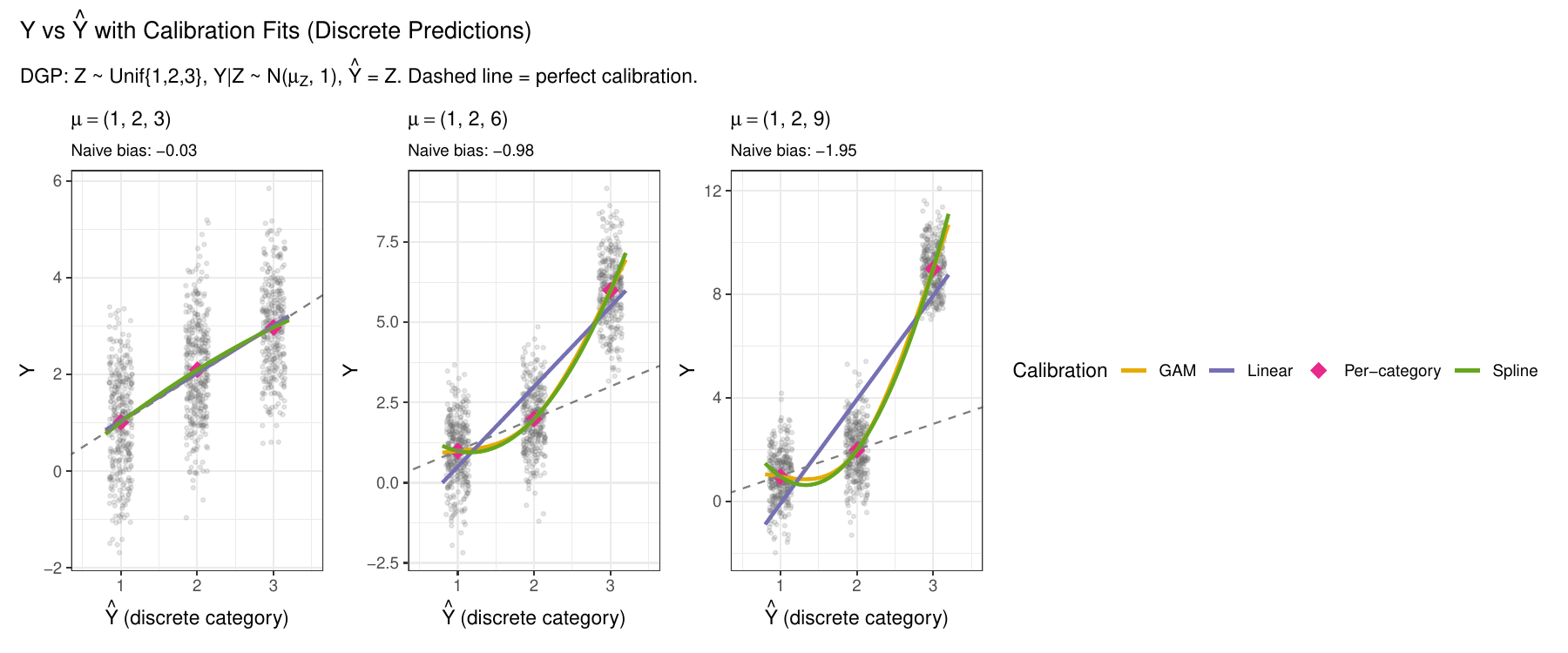}
      \caption{Calibration fits for the continuous $Y$ simulation.}
      \label{fig:calibration-fits}
  \end{figure}
\end{document}